\documentclass[runningheads]{llncs}

\usepackage{amsmath,amssymb,mathtools,mathrsfs}
\usepackage{bm}
\usepackage{microtype}
\usepackage{enumitem}
\usepackage{dsfont}
\usepackage{cleveref}

\newcommand{\C}{\mathcal{C}}

\newcommand{\Obj}{\mathrm{Ob}}
\newcommand{\Hom}{\mathrm{Hom}}
\newcommand{\EqvCont}{\mathrm{EqvCont}}
\newcommand{\CENN}{\mathsf{CENN}}
\newcommand{\Id}{\mathrm{id}}
\newcommand{\supp}{\mathrm{supp}}
\newcommand{\op}{\mathrm{op}}

\title{
Categorical Equivariant Deep Learning:\\ 
Category-Equivariant Neural Networks and Universal Approximation Theorems}
\titlerunning{Categorical Equivariant Deep Learning}

\author{Yoshihiro Maruyama\inst{1,2}}
\authorrunning{Y. Maruyama}
\institute{School of Informatics, Nagoya University, Japan\\
\email{maruyama@i.nagoya-u.ac.jp}
\and
School of Computing, Australian National University, Australia\\
\email{yoshihiro.maruyama@anu.edu.au}
}

\begin{document}
\maketitle

\begin{abstract}
We develop a theory of category-equivariant neural networks (CENNs) that unifies group/groupoid-equivariant networks, poset/lattice-equivariant networks, graph and sheaf neural networks. Equivariance is formulated as naturality in a topological category with Radon measures. Formulating linear and nonlinear layers in the categorical setup, we prove the equivariant universal approximation theorem in the general setting: the class of finite-depth CENNs is dense in the space of continuous equivariant transformations. We instantiate the framework for groups/groupoids, posets/lattices, graphs and cellular sheaves, deriving universal approximation theorems for them in a systematic manner. Categorical equivariant deep learning thus allows us to expand the horizons of equivariant deep learning beyond group actions, encompassing not only geometric symmetries but also contextual and compositional symmetries. 
\end{abstract}


\section{Introduction}\label{sec:intro}

Equivariance is a guiding principle in modern machine learning: models should respect the structural symmetries of data. Convolutional networks are translation--equivariant \cite{CohenWelling2016}, and steerable CNNs extend this idea to rotations and other continuous groups \cite{CohenWelling2017Steer,WeilerCesa2019}. Yet many domains display richer forms of structure beyond global group actions, such as symmetries that are typed, partial, or non-invertible. Monoids capture sequential or causal processes involving duplication or erasure; posets express hierarchical or lattice-ordered dependencies; and sheaves encode local-to-global consistency of data \cite{CurryThesis,HansenGhrist2019,HansenGebhart2020SNN}. These diverse structures extend far beyond the classical group paradigm but can all be described within a single unifying mathematical language: \emph{categories} \cite{MacLane}. In categorical terms, data representations are \emph{functors} assigning feature spaces to objects and transporting them along arrows, and a model is \emph{equivariant} precisely when it forms a \emph{natural transformation} between such functors, commuting with every relational or compositional structure encoded by the category. We note that group equivariance is precisely characterized by the naturality law in category theory, and our framework encompasses the theory of group-equivariant neural networks (and many others). 
Examples of categorical equivariance across various types of data in nature and society are given in the following table.

\begin{table}[!h]
\centering
\caption{Categorical equivariance across various types of data in nature and society.}
\renewcommand{\arraystretch}{1.15}
\begin{tabular}{llll}
\hline
\textbf{Category} & \textbf{Typical data} & \textbf{Arrows encode} & \textbf{Equivariance enforces}\\
\hline
Group & invertible sym. data & e.g. isometries & global sym. invariance\\
Monoid & time, sequence & causal updates & consistency through time\\
Poset & hierarchy & order relations & hierarchical consistency\\
Lattice & logic/formal concept & entailment/meet/join & logical consistency\\
Graph & network & adjacency & message--passing symmetry\\
Sheaf & spatial field & region inclusion/rels. & local--to--global consistency\\
Groupoid & multi--orbit system & local symmetries & orbitwise invariance\\
General cat. & typed relations & compositional process & naturality wrt all relations\\
\hline
\end{tabular}
\label{tab:cat-equivariance}
\end{table}

Motivated by this, we develop \emph{category-equivariant neural networks} (CENNs): 
\begin{enumerate}
\item Linear layers of CENNs are defined as category convolutions, integrals/sums over arrows with kernels constrained by naturality laws. 
Nonlinear layers of CENNs are given by scalar gates, ensuring equivariance for arrows. 
\item Our main result is the categorical equivariant universal approximation theorem (UAT): finite-depth CENNs are dense in the space of continuous equivariant transformations. 
\item We instantiate the categorical UAT for groups/groupoids, posets/lattices, graphs and cellular sheaves, thus obtaining UATs for them in a systematic manner.
\end{enumerate}

Related work is as follows. Most prior work on equivariant deep learning has focused on group actions, from group-equivariant CNNs \cite{CohenWelling2016} and steerable/filter-bank designs \cite{CohenWelling2017Steer,Worrall2017Harmonic,WeilerCesa2019} to spherical and SE(3)/E(n)–equivariant architectures \cite{CohenEtAl2018Spherical,KondorLinTrivedi2018CGNets,AndersonHyKondor2019,FuchsEtAl2020SE3} 
and E(n)-equivariant graph networks \cite{SatorrasEtAl2021En}. On graphs and sets, invariance/equivariance and universality have been developed in \cite{MaronICLR2019,KerivenPeyre2019,ZaheerEtAl2017DeepSets,LeeEtAl2019SetTransformer,XuHuLeskovecJegelka2019}. There are also sheaf-theoretic approaches \cite{CurryThesis,HansenGhrist2019}, including recent neural sheaf models \cite{HansenGebhart2020SNN,BarberoEtAl2022}. Surveys synthesize threads of equivariant learning under geometric deep learning \cite{BronsteinSPM2017,Bronstein2021}. 
In contrast, we formulate equivariance as naturality in a topological category with Radon measures and realize linear layers as category convolutions; together with scalar-gated nonlinearities and other auxiliary layers, this yields a single framework that covers groups, monoids, thin categories (posets/lattices), graphs, and cellular sheaves, and delivers a unified UAT for continuous equivariant maps, generalizing group-centric UATs such as \cite{YarotskyConstructApprox2022}.

The rest of the paper is organized as follows. 
Section~\ref{sec:cenn} defines CENNs. Section~\ref{sec:uat} states the equivariant universal approximation theorem for CENNs and outlines a proof of the theorem; full proof details are provided in the appendix.
Section~\ref{sec:conclusion} gives the concluding remarks. 
In the appendix, we also specialize the general categorical approximation theorem to groups/groupoids, posets/lattices, graphs and cellular sheaves. 

The paper intertwines categorical and measure--theoretic reasoning, which may be unfamiliar to some readers; we therefore interleave formal developments with informal commentary, illustrating the intuition behind the technical machinery.

\section{Category-Equivariant Neural Networks}\label{sec:cenn}

We develop the foundations of category-equivariant neural networks. 
For measure theory background, we refer the reader to 
\cite{Halmos,Folland,Bogachev,Rudin}.

For an arrow $f$ in $\C$, we write
$(f\circ -)$ for postcomposition $u \mapsto f\circ u$, and
$(-\circ f)$ for precomposition $u \mapsto u\circ f$. 
Pushforwards of measures along these maps are denoted $(f\circ -)_\#$ and $(-\circ f)_\#$.

\begin{definition}[Topological category with measures]\label{def:topcat}
A small category $\C$ is \emph{topological} if each hom–set $\Hom_\C(b,a)$ is equipped with a second–countable, locally compact, Hausdorff (LCH) topology and if the composition operation 
\[
\circ:\Hom_\C(b,a)\times\Hom_\C(c,b)\longrightarrow\Hom_\C(c,a)
\]
is continuous. A \emph{topological category with measures} is a topological category with a family of $\sigma$–finite Radon measures $\{\mu_{b,a}\}$ on hom–sets $\Hom_\C(b,a)$.\footnote{Note that $\mu_{b,a}$ is a measure on $\Hom_\C(b,a)$ (not on $\Hom_\C(a,b)$).} We write
\[
I(a):=\bigsqcup_{b\in\Obj\,\C}\Hom_\C(b,a),\qquad \mu_a:=\bigoplus_{b\in\Obj\,\C}\mu_{b,a}
\quad\text{on }\ I(a)
\]
and assume $\mu_a$ is a $\sigma$--finite measure.
We also assume null‑set preservation (NSP): for each arrow $w:a\to c$ and each measurable $N\subset I(c)$ with $\mu_c(N)=0$, we have $\mu_a\big(\{u\in I(a): w\circ u\in N\}\big)=0$. 
We say that $\{\mu_{b,a}\}$ is \emph{left–coherent} if
\[
(w\circ -)_\#\,\mu_{b,a}=\mu_{b,c}\qquad\text{for every arrow }w:a\to c\text{ and every object }b,
\]
and \emph{bi–coherent} if, in addition,
\[
(-\circ v)_\#\,\mu_{c,a}=\mu_{b,a}\qquad\text{for every arrow }v:b\to c\text{ and every object }a.
\]
\end{definition}
Note that we do not assume the above coherence properties in our general categorical setup; we use them to discuss measure-theoretic properties in several examples. 

We typically use measures to make relevant integrals well defined. In specializations to concrete structures, we have two tracks:
\begin{itemize}
\item \textit{Continuous track (groups/monoids/groupoids).} We typically equip hom–sets with left–coherent Radon measures (e.g.\ a left Haar measure/Haar system for groups/groupoids) to construct steerable kernels and category convolution layers; bi–coherence is usually too strong a condition. 
\item \textit{Discrete track (posets, graphs, cellular sheaves).} We typically take counting measure on each hom–set. In this case, convolution integrals are finite sums and relevant conditions reduce to pointwise incidence/restriction compatibilities; we do not use the measure coherence properties in this track.
\end{itemize}
These examples are also useful to understand the general theory developed in this paper. We discuss examples in detail in the specialization appendix.

\subsection{Feature functors and continuous equivariance}

A feature functor introduced below specifies how data are organized and transported across the objects and arrows of the category. Each object $a$ carries a local base space $\Omega(a)$ of coordinates and a fiber $E_X(a)$ (or $E_Y(a)$) of feature values, while each arrow $u:b\to a$ is provided with information about how coordinates and features are pulled back from $a$ to $b$. Functoriality ensures that these transports compose consistently, generalizing the usual notion of feature transformation under symmetries such as translations, rotations, inclusions, or message--passing relations. 

We write $s(u)$ and $t(u)$ for the source and target of an arrow $u$ in a category. Let $\mathbf{Vect}$ denote the category of real vector spaces and linear maps. $\mathcal{L}(V,W)$ denotes the space of linear maps from $V$ to $W$. $V^*$ denotes the dual space of a vector space $V$. 
We define feature functors $X,Y:\C^{op}\!\to\!\mathbf{Vect}$ as follows.

\begin{definition}[Feature functors]\label{feature}
For each object $a\in\Obj\,\C$, fix a compact base space $\Omega(a)\subset\mathbb{R}^{m_a}$ of ambient dimension $m_a\in\mathbb{N}$ and Euclidean feature fibers $E_X(a)=\mathbb{R}^{n_X(a)}$, $E_Y(a)=\mathbb{R}^{n_Y(a)}$ of dimensions $n_X(a),n_Y(a)$, respectively.\footnote{As noted above, intuitively, $\Omega(a)$ is the domain of local coordinates for object $a$, while $E_X(a),E_Y(a)$ carry feature values over that domain.}

For each arrow $u:b\to a$, we fix the following maps (assuming they exist):\footnote{They exist in the standard examples discussed in the appendix: groups/groupoids, posets/lattices and graphs/sheaves.}
\begin{itemize}[leftmargin=1.2em]
  \item \emph{continuous maps (base transports):} $\tau_u:\Omega(a)\rightarrow \Omega(b)$ and $\pi_u:\Omega(b)\rightarrow \Omega(a)$,
  \item \emph{linear maps (fiber transports):} $L^X_u: E_X(a)\rightarrow E_X(b)$ and $L^Y_u: E_Y(a)\rightarrow E_Y(b)$.
\end{itemize}
These are required to satisfy, for all composable $v:c\to b$ and $w:a\to c$,
\[
\begin{aligned}
&\text{Functoriality:} && \tau_{u\circ v}=\tau_v\circ\tau_u,\quad \pi_{u\circ v}=\pi_u\circ\pi_v,\\
&&& L^X_{u\circ v}=L^X_v\circ L^X_u,\quad L^Y_{u\circ v}=L^Y_v\circ L^Y_u,\\
&&& \tau_{\mathrm{id}_a}=\pi_{\mathrm{id}_a}=\mathrm{id},\quad L^X_{\mathrm{id}_a}=L^Y_{\mathrm{id}_a}=\mathrm{id};\\[4pt]
&\text{Mixed–base law:} && \tau_{w\circ u}\circ \pi_w=\tau_u.
\end{aligned}
\]
We assume that the maps $u\mapsto L^X_u$ and $u\mapsto L^Y_u$ are continuous on $\Hom_\C(b,a)$, and $(u,y)\mapsto \tau_u(y)$ and $(u,x)\mapsto \pi_u(x)$ are continuous on $\Hom_\C(b,a)\times\Omega(a)$ and $\Hom_\C(b,a)\times\Omega(b)$, respectively.\footnote{In particular, $(u,y)\mapsto \pi_u(\tau_u y)$ is continuous on $\Hom_\C(b,a)\times\Omega(a)$.}
We finally define feature functors:
\[
\begin{aligned}
X(a)&:=C(\Omega(a),E_X(a)),&
\quad X(u)(f_a)&:=L^X_u\circ f_a\circ \pi_u\in C(\Omega(b),E_X(b)),\\
Y(a)&:=C(\Omega(a),E_Y(a)),&
\quad Y(u)(g_a)&:=L^Y_u\circ g_a\circ \pi_u\in C(\Omega(b),E_Y(b)).
\end{aligned}
\]
\end{definition}

\begin{definition}[Continuous category-equivariant maps]\label{def:eqvcont}
Let $X,Y:\C^{op}\!\to\!\mathbf{Vect}$ be feature functors. A \emph{continuous category--equivariant map} (continuous natural transformation) from $X$ to $Y$ is a family
\[
\Phi=\{\Phi_a:X(a)\to Y(a)\}_{a\in\Obj\,\C}
\]
such that:
\begin{itemize}
  \item each $\Phi_a$ is continuous for the sup--norm topologies on
        $X(a)=C(\Omega(a),E_X(a))$ and $Y(a)=C(\Omega(a),E_Y(a))$; 
  \item for every arrow $w:a\to c$, the \emph{naturality} condition holds:
        \begin{equation}\label{eq:naturality-contravariant}
          Y(w)\circ \Phi_c \;=\; \Phi_a\circ X(w).
        \end{equation}
\end{itemize}
We denote the space of such $\Phi$ by $\EqvCont(X,Y)$. For a finite $F\subset\Obj\,\C$ and, for each $a\in F$, a compact set $K_a\subset X(a)$ (compact in the sup--norm topology), define the seminorm
\[
\|\Phi\|_{(K_a)_{a\in F},\,F}
\;:=\;
\max_{a\in F}\ \sup_{f\in K_a}\ \|\Phi_a(f)\|_{\infty}.
\]
The locally convex topology generated by these seminorms is called the
\emph{compact--open, finite--object topology} on $\EqvCont(X,Y)$.
\end{definition}

With the above topology, $\EqvCont(X,Y)$ becomes a locally convex vector space under pointwise addition and scalar multiplication.


\subsection{Linear layers: category convolutions}

We introduce category convolutions, which integrate/sum information along incoming arrows using a kernel constrained by an integrated naturality law. Intuitively, a \emph{category kernel} introduced below specifies, for each arrow and base point, how features are transported from the source object to the target object, generalizing the role of convolutional filters so that they respect the compositional structure of the category. The naturality law makes the resulting convolution layer a natural transformation, ensuring that feature updates remain equivariant with respect to the compositional structure of the category.

\begin{definition}[Category kernel]
\label{def:natk}
Let $Z,Z':\C^{op}\!\to\!\mathbf{Vect}$ be feature functors. 
A \emph{category kernel} from $Z$ to $Z'$ is a family\footnote{We omit the subscript of a category kernel when typing is clear.}
\[
\mathsf K=\{\mathsf K_{b\to a}\}_{(b,a)\in \Obj\,\C \times \Obj\,\C},\quad
\mathsf K_{b\to a}:\ \Hom_\C(b,a)\times\Omega(a)\longrightarrow
\mathcal L \big(E_Z(b),E_{Z'}(a)\big),
\]
satisfying:

\smallskip
\noindent\textit{(C) Carath\'eodory regularity.}
For each fixed $u\in\Hom_\C(b,a)$, the map $y\mapsto\mathsf K_{b\to a}(u,y)$ is continuous;
for each fixed $y\in\Omega(a)$, the map $u\mapsto \mathsf K_{b\to a}(u,y)$ is Borel measurable
on $\Hom_\C(b,a)$.

\noindent\textit{(IN) Integrated naturality.}
For every arrow $w:a\to c$, every $y\in\Omega(a)$, and every $x_c\in Z(c)$,\footnote{Intuitively, this identity expresses equivariance of convolution under composition. The kernel~$\mathsf K$ is ``natural up to integration'', ensuring that category convolutions commute with morphism composition. We also note the typing of formulae in this equation: $\pi_w:\Omega(a)\!\to\!\Omega(c)$ and $\tau_{u'}:\Omega(c)\!\to\!\Omega(s(u'))$ make $\tau_{u'}\pi_w y\in\Omega(s(u'))$ well--typed; $Z(u')x_c\in Z(s(u'))$ is then evaluable at $\tau_{u'}\pi_w y$. On the right, $Z(w\circ u):Z(c)\!\to\!Z(s(u))$ is contravariant and evaluable at $\tau_u y\in\Omega(s(u))$.}
\begin{multline}\label{eq:natk-integrated}
L^{Z'}_{w}\!\left(
  \int_{I(c)}
    \mathsf K_{s(u')\to c}\!\big(u',\,\pi_{w} y\big)\,
    \big(Z(u')\,x_c\big)\!\big(\tau_{u'}\,\pi_{w} y\big)\, d\mu_{c}(u')
\right)\\
=\;
\int_{I(a)}
  \mathsf K_{s(u)\to a}(u,\,y)\,
  \big(Z(w\!\circ\!u)\,x_c\big)\!\big(\tau_{u} y\big)\, d\mu_{a}(u).
\end{multline}

\noindent\textit{(L\textsuperscript{1}) Uniform integrable bound.}
For each object $a$, there exists $G_a\in L^1(\mu_a)$ such that
\[
\sup_{y\in\Omega(a)}\ \big\|\mathsf K_{s(u)\to a}(u,y)\big\|\ \le\ G_a(u)
\qquad\text{for $\mu_a$--a.e.\ $u\in I(a)$.}
\]
\end{definition}

For a feature functor $Z':\C^{op}\!\to\!\mathbf{Vect}$,
a \emph{natural bias} is defined as a family $b=\{b_a\}_{a\in\Obj\,\C}$ with
$b_a\in C(\Omega(a),E_{Z'}(a))$ such that for every arrow $w:a\to c$,
\[
Z'(w)\,b_c \;=\; b_a.
\]

\begin{definition}[Category convolution]\label{def:catconv}
Let $Z,Z':\mathcal C^{op}\!\to\!\mathbf{Vect}$ be feature functors, $\mathsf K$ a category kernel, and $b=\{b_a\}_{a\in\mathrm{Ob}\,\mathcal C}$ a natural bias for $Z'$. 
The \emph{category convolution} associated with $\mathsf K$ is the family
$\widetilde L_{\mathsf K}=\{(\widetilde L_{\mathsf K})_a\}_{a\in\Obj\,\C}$ defined, for
$y\in\Omega(a)$ and $x_a\in Z(a)=C(\Omega(a),E_Z(a))$, by
\begin{equation}\label{eq:catconv}
(\widetilde L_{\mathsf K} x)_a(y)
\;:=\;
b_a(y)
\;+\;
\int_{I(a)}
  \mathsf K_{\,s(u)\to a}\big(u,y\big)\;
  \big(Z(u)\,x_a\big)\!\big(\tau_{u} y\big)\;
  d\mu_a(u),
\end{equation}
where $(Z(u)\,x_a)(\tau_u y)=L^Z_{u}\big(x_a(\pi_u\tau_u y)\big)\in E_Z(s(u))$.
\end{definition}

Intuitively, a category convolution aggregates information by integrating features that have been transported along all incoming arrows. The forward base maps $\pi_u:\Omega(s(u))\to \Omega(t(u))$ determine how each feature is precomposed for the functorial action, while the pullback maps $\tau_u:\Omega(t(u)) \to \Omega(s(u))$ specify where those transported features are evaluated in the integral.  In this way, category convolution generalizes the classical notion of convolution: it accumulates information over all morphisms entering an object so that feature updates respect the compositional structure of the category.

In the rest of the paper, we assume that, for each object \(a\), the transport norms are essentially bounded:
\begin{equation}\label{eq:ess-bdd-LZ}
\|L^Z_u\| \in L^\infty(I(a),\mu_a),
\ \ \text{i.e.,}\ \ 
\|L^Z_u\| \le H_a \ \text{for }\mu_a\text{-a.e. }u\in I(a)
\end{equation}
for some finite constant \(H_a<\infty\). 

\subsection{CENN: Category-equivariant neural networks}

We first introduce nonlinearities. 

\begin{definition}[Scalar–gated nonlinearity]\label{def:gate}
Let $Z:\C^{op}\!\to\!\mathbf{Vect}$ be a feature functor. 
Define the scalar functor $S:\C^{op}\!\to\!\mathbf{Vect}$ by
\[
S(a)=C(\Omega(a),\mathbb R),\qquad S(u)(r_a)=r_a\circ \pi_u\quad (u:b\to a).
\]
A \emph{natural scalar channel} is a natural transformation
$s:Z\Rightarrow S$.
Fix a continuous activation $\alpha:\mathbb R\to\mathbb R$; 
throughout, we assume that the activation is globally Lipschitz.
The \emph{scalar–gated nonlinearity} associated with $(\alpha,s)$ is
the family $\Sigma^{\alpha,s}=\{\Sigma^{\alpha,s}_a:Z(a)\to Z(a)\}_{a\in \Obj\,\C}$ acting
pointwise by
\[
(\Sigma^{\alpha,s}_a z_a)(y)\;:=\;\alpha\big(s_a(z_a)(y)\big)\,z_a(y),
\qquad z_a\in Z(a),\ y\in\Omega(a).
\]
\end{definition}

We now introduce two auxiliary functors for the following development.

\begin{definition}[Arrow–bundle functor and arrow–bundle lifts]\label{def:arrow-bundle}
Let $Z:\C^{op}\!\to\!\mathbf{Vect}$ be a feature functor. 
For each object $a$, define the \emph{arrow–bundle} of $Z$ by
\[
Z_{\downarrow}(a)
\;:=\;
\Bigl\{(h_u)_{u\in I(a)} \ \big|\ h_u\in Z\!\big(s(u)\big),\ 
\operatorname*{ess\,sup}_{u\in I(a)}\|h_u\|_{\infty}<\infty\Bigr\}\Big/\!\sim,
\]
where $(h_u)\sim(h'_u)$ iff $\|h_u-h'_u\|_\infty=0$ for $\mu_a$–a.e.\ $u$. Equip it with the essential‑sup norm
\[
\|(h_u)_{u\in I(a)}\|_{Z_{\downarrow}(a)}
\;:=\;
\operatorname*{ess\,sup}_{u\in I(a)}\|h_u\|_{\infty}.
\]
For each arrow $w:a\to c$, define the reindexing map
\[
Z_{\downarrow}(w):Z_{\downarrow}(c)\longrightarrow Z_{\downarrow}(a),
\qquad
\big(Z_{\downarrow}(w)(h_{u'})_{u'\in I(c)}\big)_u\ :=\ h_{w\circ u}\in Z\!\big(s(u)\big).
\]

The \emph{arrow–bundle lift} is the natural transformation
\[
\Delta_Z:Z\Rightarrow Z_{\downarrow},
\qquad
(\Delta_Z)_a(h)\ :=\ \big(Z(u)\,h\big)_{u\in I(a)}.
\]

More generally, if $Z':\C^{op}\!\to\!\mathbf{Vect}$ and 
$H=\{H_c:Z(c)\to Z'(c)\}_{c\in \Obj\,\C}$ is a family of objectwise continuous maps,
its \emph{componentwise lift} is the natural transformation
\[
H_{\downarrow}:Z_{\downarrow}\Rightarrow Z'_{\downarrow},
\qquad
\big(H_{\downarrow}\big)_a\big((x_u)_{u\in I(a)}\big)\ :=\ \big(H_{s(u)}(x_u)\big)_{u\in I(a)}.
\]
\end{definition}

Proofs for well-definedness of these functors are provided in the appendix.

\begin{definition}[Admissible componentwise lifts]\label{def:adm-lift}
Let $H=\{H_c: Z(c)\to Z'(c)\}_{c\in \Obj\,\C}$ be a family of objectwise continuous maps.
We say that $H$ is \emph{arrow–bundle uniformly Lipschitz} if for each 
$a\in \mathrm{Ob}\,C$ there exists $L_a<\infty$ such that, for $\mu_a$–a.e.\ $u\in I(a)$ and all 
$x,x'\in Z(s(u))$,
\[
\|H_{s(u)}(x)-H_{s(u)}(x')\|_{\infty}\ \le\ L_a\,\|x-x'\|_{\infty}.
\]
Equivalently, $L_a=\operatorname*{ess\,sup}_{u\in I(a)} \mathrm{Lip}\big(H_{s(u)}\big)<\infty$.\footnote{$\mathrm{Lip}(T)$ denotes the global Lipschitz constant of $T$.}
In this case, the componentwise lift $H_{\downarrow}:Z_{\downarrow}\Rightarrow Z'_{\downarrow}$ defined by
\[
\big(H_{\downarrow}\big)_a\!\big((x_u)_{u\in I(a)}\big)=\big(H_{s(u)}(x_u)\big)_{u\in I(a)}
\]
is called an \emph{admissible componentwise lift}. 
\end{definition}

We introduce arrow–bundle convolution and arrow–bundle integrated naturality.
Let $Z,Z' : C^{\mathrm{op}}\!\to\!\bf{Vect}$ be feature functors and let a kernel
\(
\mathsf K_{s(u)\to a}(u,y)\in L \big(E_Z(s(u)),\,E_{Z'}(a)\big)
\)
satisfy the Carathéodory regularity and uniform $L^1$ bound conditions on $I(a)$.

\begin{definition}[Arrow–bundle convolution]
For $x=(x_u)_{u\in I(a)}\in Z_\downarrow(a)$ and a natural bias $b$ for $Z'$, define the
\emph{arrow–bundle convolution}
\[
\big(L^\downarrow_{\mathsf K} x\big)_a(y)
  := b_a(y)\;+\!\int_{I(a)} \mathsf K_{s(u)\to a}(u,y)\; x_u(\tau_u y)\,d\mu_a(u).
\]
We say that $\mathsf K$ satisfies \emph{arrow–bundle integrated naturality} $(\mathrm{IN}_{\downarrow})$ if,
for every $w:a\to c$, every $y\in\Omega(a)$, and every arrow–family
$f=(f_{u'})_{u'\in I(c)}$ with $f_{u'}\in Z(s(u'))$,
\[
L^{Z'}_{w}\!\int_{I(c)} \mathsf K(u',\pi_w y)\, f_{u'}(\tau_{u'}\pi_w y)\,d\mu_c(u')
\;=\;
\int_{I(a)} \mathsf K(u,y)\, f_{w\circ u}(\tau_u y)\,d\mu_a(u).
\tag{$\mathrm{IN}_{\downarrow}$}
\]
\end{definition}

Under (C), (L1), and $(\mathrm{IN}_{\downarrow})$, the map $L^\downarrow_{\mathsf K}:Z_\downarrow\Rightarrow Z'$ is a
continuous natural transformation (proof details are provided in the appendix).

We are ready to define category-equivariant neural networks.

\begin{definition}[Category–equivariant neural network (CENN)]\label{def:cenn}
Let $X,Y:\C^{op}\!\to\!\mathbf{Vect}$ be feature functors and
fix a nonpolynomial, continuous, globally Lipschitz activation $\alpha:\mathbb R\to\mathbb R$.
A \emph{category–equivariant neural network (CENN)} from $X$ to $Y$ is a
finite composition of continuous natural transformations drawn from the
following admissible layer types:
\begin{enumerate}
\item \emph{Category convolutions:}
$\widetilde L_{\mathsf K}:Z\Rightarrow Z'$ associated with a category kernel
$\mathsf K$ (satisfying \textup{(C)}, \textup{(IN)}, \textup{(L1)}) and a natural bias $b$;
\item \emph{Scalar–gated nonlinearities:}
$\Sigma^{\alpha,s}:Z'\Rightarrow Z'$ where $s:Z'\Rightarrow S$ is a natural
scalar channel defined above;
\item \emph{Arrow–bundle lifts and admissible componentwise lifts:}
$\Delta_Z:Z\Rightarrow Z_{\downarrow}$ and 
$H_{\downarrow}:Z_{\downarrow}\Rightarrow Z'_{\downarrow}$ defined above; 
\item \emph{Arrow–bundle convolutions:}
$L^{\downarrow}_{\mathsf K}:Z_{\downarrow}\Rightarrow Z'$
associated with an arrow–bundle kernel $\mathsf K_{\,s(u)\to a}(u,y)$ defined above, 
together with a natural bias $b$ for $Z'$.
\end{enumerate}
A depth–$L$ CENN has intermediate functors
$Z^{(k)}:\C^{op}\!\to\!\mathbf{Vect}$ (each $Z^{(k)}$ may itself be an arrow–bundle
of a previous functor) with
\[
Z^{(0)}=X,\qquad Z^{(L)}=Y,
\]
and can be written as
\[
\Phi
\;=\;
\mathcal L_{L}\circ \mathcal L_{L-1}\circ\cdots\circ \mathcal L_{1},
\]
where each $\mathcal L_k$ is one of the admissible layers 1–4 with
coherently typed domains and codomains.
We denote the family of all such $\Phi$ by $\CENN_\alpha(X,Y)$.
\end{definition}



\section{Category-Equivariant Universal Approximation}\label{sec:uat}

We formulate the universal approximation theorem for category-equivariant neural networks; we need some preparations.

A \emph{probe family} is a choice, for each arrow $u:d\to c$ in $\mathcal C$, of a map
\[
  \sigma_u:\Omega(c)\longrightarrow\Omega(d),
\]
such that the following hold: (i) $\sigma_{\mathrm{id}_a}=\mathrm{id}_{\Omega(a)}$ for each object $a$ and (ii) the map $(u,y)\mapsto \sigma_u(y)$ is continuous on $\mathrm{Hom}_{\mathcal C}(d,c)\times\Omega(c)$.
Note that $\sigma$ is similar to $\tau$, but we require another property for $\sigma$ as a probe family for separation as we define in the next paragraph.\footnote{We note that, even for the separation property, $\tau$ itself works as $\sigma$ in examples such as thin/discrete categories; however, $\sigma$ cannot be $\tau$ in some other cases.}


Fix $a$ and a compact $K_a\subset X(a)$. For each $y\in\Omega(a)$ we assume the arrow–evaluation separation at $y$: If $x_a\neq x'_a$ in $K_a$, then there exists $u\in I(a)$ such that
\[
(X(u)x_a)\big(\sigma_u y\big) \neq (X(u)x'_a)\big(\sigma_u y\big).
\]
We fix a probe family $\sigma$ satisfying this property. We note that this holds in standard examples.\footnote{We instantiate $\sigma_u$ as follows: $\sigma_u=\mathrm{id}$ for compact groups/groupoids; and $\sigma_u=\tau_u$ for posets, graphs and cellular sheaves.}

For the probe family $\sigma$, given a kernel $\mathsf K$, define
\[
\big(L^{\downarrow,\sigma}_{\mathsf K} x\big)_a(y)
:= b_a(y)+\int_{I(a)} \mathsf K(u,y)\,x_u\big(\sigma_u y\big)\,d\mu_a(u).
\]
We say $\mathsf K$ satisfies \emph{probe–integrated naturality} \textup{(IN$^\downarrow_\sigma$)} if, for every
$w:a\to c$ and $y\in\Omega(a)$,
\[
L_{Z'}(w)\!\int_{I(c)}\!\mathsf K(u',\pi_w y)\,f_{u'}\!\big(\sigma_{u'}\pi_w y\big)\,d\mu_c(u')
\;=\;\int_{I(a)}\!\mathsf K(u,y)\,f_{w\circ u}\!\big(\sigma_u y\big)\,d\mu_a(u)
\]
for all arrow–families $f=(f_{u'})$. 
Under \textup{(C)}, \textup{(L$^1$)}, and \textup{(IN$^\downarrow_\sigma$)}, $L^{\downarrow,\sigma}_{\mathsf K}:Z_\downarrow\Rightarrow Z'$ is a continuous natural transformation.





We assume the shrinking–support approximate identity condition (SI): For each object $a$ and each arrow $u_0\in I(a)$ there exists a family $\{\kappa^{(\varepsilon)}_{a,u_0}\}_{\varepsilon>0}\subset L^1(I(a),\mu_a)$ with $\kappa^{(\varepsilon)}_{a,u_0}\ge 0$, $\int\kappa^{(\varepsilon)}_{a,u_0}\,d\mu_a=1$,
$\operatorname{supp}\kappa^{(\varepsilon)}_{a,u_0}\subset U_\varepsilon(u_0)$ where
$U_\varepsilon(u_0)$ is a neighbourhood base shrinking to $u_0$, and
\[
\int_{I(a)}\!f(u)\,\kappa^{(\varepsilon)}_{a,u_0}(u)\,d\mu_a(u)\ \xrightarrow[\varepsilon\to 0]{}\ f(u_0)
\quad\text{for every bounded continuous }f\text{ on }I(a).
\]
Note that such a family exists in standard examples.\footnote{For a compact group $G$ with left Haar measure $\mu$, take $\kappa^{(\varepsilon)}_{g_0}(g):=\eta_\varepsilon(g_0^{-1}g)$ with $\eta_\varepsilon\in C_c(G)$ (space of continuous functions with compact support), $\eta_\varepsilon\ge0$, $\int\eta_\varepsilon\,d\mu=1$, $\operatorname{supp}\eta_\varepsilon\subset U_\varepsilon(e)$ where $U_\varepsilon(e) \downarrow\{e\}$ is a neighbourhood base at $e$; then $\int f(g)\,\kappa^{(\varepsilon)}_{g_0}(g)\,d\mu(g)\to f(g_0)$. In discrete/finite settings, we may simply take $\kappa^{(\varepsilon)}_{a,u_0}=\mathbf 1_{\{u_0\}}$ (which denotes the indicator function).}

We also assume the equivariant compilation assumption (EC):\footnote{For compact groups with Haar measure (normalized to $1$), one may take
\(
(Rx)_a(y)\;=\;\int_{u\in I(a)} L^Y_{u^{-1}}\,x_u(\tau_u y)\,d\mu_a(u),
\)
which is an arrow–bundle convolution satisfying ($\mathrm{IN}_\downarrow$) and $R\circ\Delta_Y=\mathrm{id}_Y$ by construction. We discuss other examples in the appendix.}
There exists a continuous natural transformation
\[
R:\;Y_{\downarrow}\Rightarrow Y,\qquad R\circ\Delta_Y=\mathrm{id}_Y,
\]
realized as an arrow–bundle convolution $L^{\downarrow}_{\mathsf R}$ with zero natural bias, 
whose kernel $\mathsf R_{\,s(u)\to a}(u,y)\in L \big(E_Y(s(u)),E_Y(a)\big)$
satisfies (C), ($\mathrm{IN}_{\downarrow}$), (L1). 
Concretely, for $x=(x_u)_{u\in I(a)}\in Y_{\downarrow}(a)$,
\[
\big(L^{\downarrow}_{\mathsf R}x\big)_a(y)\;=\;\int_{I(a)}\mathsf R_{\,s(u)\to a}(u,y)\,x_u(\tau_u y)\,d\mu_a(u).
\]

We now state the equivariant universal approximation theorem.


\begin{theorem}[Equivariant universal approximation theorem]\label{thm:density}
The class $\CENN_\alpha(X,Y)$ of finite–depth category–equivariant neural networks is dense in $\EqvCont(X,Y)$ in the compact–open, finite–object topology.
Equivalently, for every $\Phi\in\EqvCont(X,Y)$, $\varepsilon>0$, finite
$F\subset\Obj\,\C$, and compact sets $K_a\subset X(a)$ ($a\in F$), there exists
$\Psi\in\CENN_\alpha(X,Y)$ such that
\[
\max_{a\in F}\ \sup_{x\in K_a}\ \big\|\,\Phi_a(x)-\Psi_a(x)\,\big\|_{\infty}
\;<\;\varepsilon.
\]
\end{theorem}

We give a proof outline here; full proof details are in the appendix. 

\paragraph{Proof outline of Theorem~\ref{thm:density}.}
We outline the construction of a CENN that uniformly approximates a given continuous
equivariant map $\Phi\in\EqvCont(X,Y)$.

\emph{Step 0 (Reduction to finitely many scalar coordinates).}
For each $a\in F$ choose a basis $\{v_{a,i}\}_{i=1}^{q_a}$ of $E_Y(a)$ and dual
$\{\lambda_{a,i}\}\subset E_Y(a)^{\!*}$. Approximating $\Phi_a$ in sup–norm on
$K_a\times\Omega(a)$ is equivalent (up to a uniform constant depending only on the basis)
to approximating the scalar coordinates
\[
f_{a,i}(x_a,y):=\big\langle\lambda_{a,i},\,\Phi_a(x_a)(y)\big\rangle\in\mathbb R\qquad(1\le i\le q_a).
\]

\emph{Step 1 (Stone--Weierstrass separation on $K_a\times\Omega(a)$).}
For $u\in I(a)$, $\ell\in E_X(s(u))^{*}$, and $\eta\in C(\Omega(a))$, define the
\emph{carriers}
\[
\phi_{u,\ell,\eta}(x_a,y):=\eta(y)\,\big\langle \ell,\,(X(u)x_a)(\sigma_u y)\big\rangle .
\]
Let $A_a$ be the unital real subalgebra of $C(K_a\times\Omega(a),\mathbb R)$ generated by all
$\phi_{u,\ell,\eta}$ together with the pure base weights $\psi(y)$, $\psi\in C(\Omega(a))$.
Carriers are continuous by joint continuity of $(\pi,\sigma)$ and continuity of $x\mapsto L^X_u\circ x\circ\pi_u$.
By the arrow--evaluation separation, 
$A_a$ separates points of $K_a\times\Omega(a)$; hence the Stone--Weierstrass theorem yields
\[
\overline{A_a}^{\|\cdot\|_\infty}=C(K_a\times\Omega(a),\mathbb R).
\]
Thus each scalar target $f_{a,i}$ is uniformly approximable by a polynomial in finitely many carriers.

\emph{Step 2 (Realizing carriers by convolution).}
Fix $u_0\in I(a)$ and $\ell\in E_X(s(u_0))^{\!*}$. Using (SI)
on $I(a)$, construct
\[
h_a(x_a)(y)
:=\int_{I(a)}\kappa^{(\varepsilon)}_{a,u_0}(u)\,
\big\langle \ell,\,(X(u)x_a)(\sigma_u y)\big\rangle\,d\mu_a(u).
\]
Local boundedness of transports and joint continuity give equicontinuity on $K_a\times\Omega(a)$,
while (SI) implies $h_a\!\to \phi_{u_0,\ell,1}$ uniformly on $K_a\times\Omega(a)$.
Hence each carrier is uniformly $\varepsilon$–approximated (on $K_a\times\Omega(a)$) by the $\sigma$–probe arrow–bundle convolution $(L^{\downarrow,\sigma}_{\mathsf K^{(\varepsilon)}}\!\circ\Delta_X)_a$.

\emph{Step 3 (Finite linear combinations via convolution).}
Stack finitely many realized scalar channels $h_j:X(a)\to C(\Omega(a),\mathbb R)$ from Step~2.
Identity–supported scalar kernels (plus bias) implement affine combinations
$\sum_j c_j h_j + b$; thus finite affine spans of carriers are realized by finite compositions of
scalar convolutions.

\emph{Step 4 (Pointwise nonlinearity via gates).}
Given a stacked scalar tuple $h=(h_1,\dots,h_m)$ with compact range on $K_a\times\Omega(a)$,
simulate a pointwise MLP $N:\mathbb R^m\to\mathbb R$ (affine maps + nonpolynomial $\alpha$) as a
finite composition of identity–supported scalar convolutions (Step~3) and gates
$\Sigma^{\alpha,s}$. The classical MLP UAT on compact sets then yields uniform approximation of any continuous function of the finitely many carriers.

\emph{Step 5 (Approximating the scalar coordinates).}
By Step~1 each $f_{a,i}$ is uniformly approximable by a polynomial in carriers; by Steps~2–4 each
such polynomial is realized to arbitrary accuracy by a finite CENN block acting at object $a$.
Assemble the $q_a$ scalar outputs back into $Y(a)$ using the fixed basis $\{v_{a,i}\}$.

\emph{Step 6 (Compiled equivariance via natural retraction).} By (EC), there exists a continuous natural retraction \(R:Y_{\downarrow}\Rightarrow Y\) with \(R\circ\Delta_Y=\mathrm{id}_Y\), realized by an arrow–bundle convolution \(L^{\downarrow}_{\mathsf R}\) whose kernel satisfies \textup{(C)}, \textup{(IN$^\downarrow$)}, \textup{(L1)}. For the objectwise family \(G^{(n)}=\{G^{(n)}_b\}\) from Step~5 (built by Lipschitz extension on transported compacts), apply the \(1\)–Lipschitz metric projection \(\operatorname{proj}_{B^\star}\) in each fiber with \(B^\star:=1+\max_{a\in F}H_a^Y C_a\) (here \(H_a^Y:=\operatorname*{ess\,sup}_{u\in I(a)}\|L^Y_u\|\), \(C_a:=\sup_{x\in K_a}\|\Phi_a(x)\|_\infty\)); this projection is inactive on the transported compacts for \(\delta\le 1\) and preserves arrow–bundle uniform Lipschitzness. Compile 
\[
\Psi:=L^{\downarrow}_{\mathsf R}\circ (G^{(n)})_{\downarrow}\circ\Delta_X\in\CENN_\alpha(X,Y).
\] 
Using localized stability, we have 
\(
\|\Psi-\Phi\|_{(K_a),F}\le \delta\,M_R(F)+\widehat B\,t_n
\) 
where 
\(
M_R(F):=\max_{a\in F}\int_{I(a)}G_a^{\mathsf R}\,d\mu_a,
\) 
\(t_n:=\max_{a\in F}\int_{I(a)\setminus U_{a,n}}G_a^{\mathsf R}\,d\mu_a\), and \(\widehat B:=\max_{a\in F}(B^\star+H_a^Y C_a)\);  
note that the first term in the inequality comes from approximating on a finite measure slice of arrows; the second term is a tail that vanishes by $L^1$ integrability.
Choosing \(n\) so that \(\widehat B\,t_n<\varepsilon/2\) and then \(\delta:=\min\{\varepsilon/(2M_R(F)),1\}\) yields \(\|\Psi-\Phi\|_{(K_a),F}<\varepsilon\).

\emph{Step 7 (Topology and conclusion).}
All layers above are continuous and equivariant; errors are measured in the compact–open,
finite–object topology. Therefore, for arbitrary $F,K_a,\varepsilon$, there exists
$\Psi\in\CENN_\alpha(X,Y)$ with
$$\max_{a\in F}\sup_{x_a\in K_a}\|\Phi_a(x_a)-\Psi_a(x_a)\|_{\infty}<\varepsilon,$$
proving density of $\CENN_\alpha(X,Y)$ in $\EqvCont(X,Y)$.

\section{Conclusion}\label{sec:conclusion}

We formulated category-equivariant neural networks (CENNs) and proved the equivariant universal approximation theorem for them: finite-depth CENNs are dense in the space of equivariant continuous transformations. The framework unifies group/groupoid-, poset/lattice- and graph/sheaf-based models and yields universal approximation theorems for them as corollaries of the general theorem in the categorical setup (for these examples, see the appendix below). 

We finally remark on the applicability of the framework developed here. The categorical equivariant deep learning framework provides a unified language for modeling data having diverse forms of symmetry, relation, or compositional structure. By representing these structures as categories, the framework generalizes the classical notion of equivariance beyond group actions to any system of morphisms capturing consistent transformations among data objects. Different categories formalize different manifestations of symmetry: groups encode global geometric invariances such as rotations or translations; monoids describe irreversible or causal transformations in sequential or temporal data; posets and lattices capture hierarchical and logical dependencies as in taxonomies, ontologies, or formal concept analysis; graphs represent relational symmetries underlying message--passing and network interactions; and sheaves express local--to--global coherence of spatial or topological fields. Within this setting, data representations are functors assigning feature spaces to objects, and neural layers are natural transformations that commute with all structural transports defined by the category. This categorical formulation thus subsumes a wide spectrum of architectures, from convolutional and group--equivariant networks to sheaf and graph neural networks, under a single theoretical principle, enabling the systematic design of symmetry--aware and compositional structure--aware learning models across geometry, logic, dynamics, and complex relational domains.

The categorical framework thus opens concrete pathways for developing symmetry--aware and structure--aware architectures across a broad range of scientific and technological domains.
In physics and chemistry, groupoid-- or sheaf--based CENNs can model systems with local coordinate frames, overlapping patches, or multiscale coupling between molecular or continuum components. In robotics and spatiotemporal perception, monoid-- and sheaf--equivariant networks enable representations that respect causal progression in time and spatial composition of sensor fields, improving generalization across trajectories and sensor modalities. In graph learning and relational reasoning, poset-- and lattice--equivariant constructions provide a principled way to encode logical hierarchies or entailment relations among entities, supporting tasks in knowledge representation, formal concept analysis, and symbolic reasoning. Since all these models are instances of the same categorical recipe, i.e. data representations as functors and neural layers as natural transformations, they can share implementation and optimization infrastructure while preserving theoretical guarantees such as compositional consistency and symmetry preservation. This categorical uniformity also facilitates transfer of architectural design ideas across domains, allows the modular combination of different symmetry types (e.g., various types of compositional symmetry such as temporal $\times$ spatial $\times$ logical; various types of contextual symmetry such as a family of group symmetries varying over context categories), and provides a rigorous language for discovering new equivariant architectures tailored to datasets with intrinsic structures. 

Concrete applications and empirical evaluations of the categorical equivariant framework are presented in several companion papers, where we instantiate the abstract framework for specific domains and benchmark them against existing architectures. Across diverse settings, categorical equivariance yields systematic performance improvements under distribution shifts and structural perturbations, confirming the theoretical advantage of compositional symmetry and structure preservation. The companion results demonstrate that the categorical formalism can generate effective models across distinct problem types while offering interpretability through explicit compositional structure. Together with the present theoretical work, the experimental studies establish categorical equivariant deep learning as both a unifying mathematical framework and a practically effective methodology for symmetry--aware and structure--aware machine learning.\footnote{The first experimental results are provided in \cite{Maruyama2025CatEquivRepHAR,Maruyama2025CatEquivHAR}, which employ relatively simple categorical symmetry settings yet already demonstrate clear performance gains enabled by categorical equivariance. More results are presented in subsequent papers.}

\appendix

\section{Proof of the Universal Approximation Theorem}

We provide full proof details for the equivariant universal approximation theorem.

We first prove several basic lemmas. 

\begin{lemma}[Gates are equivariant]\label{prop:gate-general}
$\Sigma^{\alpha,s}:Z\Rightarrow Z$ is a continuous natural (equivariant) transformation.  
Hence $\Sigma^{\alpha,s}\in\EqvCont(Z,Z)$.
\end{lemma}

\begin{proof}
\emph{Naturality.}
Let $u:b\to a$, $z_a\in Z(a)$, and $y\in\Omega(b)$. Using the contravariant action
$Z(u)(h_a)=L^Z_u\circ h_a\circ \pi_u$ and the naturality of $s$,
$S(u)\,s_a=s_b\,Z(u)$, we compute
\begin{align*}
\bigl(Z(u)\,\Sigma^{\alpha,s}_a z_a\bigr)(y)
&= L^Z_u\!\left(
     \alpha\!\bigl(s_a(z_a)(\pi_u y)\bigr)\,
     z_a(\pi_u y)
   \right) \\
&= \alpha\!\bigl((S(u)\,s_a(z_a))(y)\bigr)\,
   \bigl(Z(u)z_a\bigr)(y) \\
&= \alpha\!\bigl(s_b(Z(u)z_a)(y)\bigr)\,
   \bigl(Z(u)z_a\bigr)(y) \\
&= \bigl(\Sigma^{\alpha,s}_b(Z(u)z_a)\bigr)(y).
\end{align*}
so $Z(u)\circ \Sigma^{\alpha,s}_a=\Sigma^{\alpha,s}_b\circ Z(u)$.

\smallskip
\emph{Continuity.}
For fixed $a$, the map $s_a:Z(a)\to S(a)=C(\Omega(a),\mathbb R)$ is continuous by
assumption. The Nemytskii operator $C(\Omega(a),\mathbb R)\to C(\Omega(a),\mathbb R)$,
$r\mapsto \alpha\circ r$, is continuous because $\Omega(a)$ is compact and $\alpha$
is continuous. The pointwise multiplication
\[
m_a:\ S(a)\times Z(a)\longrightarrow Z(a),\qquad
(r,z)\mapsto \big(y\mapsto r(y)\,z(y)\big),
\]
is continuous bilinear and satisfies $\|m_a(r,z)\|_\infty\le\|r\|_\infty\|z\|_\infty$.
Hence $\Sigma^{\alpha,s}_a=m_a\circ\big((\alpha\circ s_a),\mathrm{id}_{Z(a)}\big)$
is continuous. This shows each component is continuous; together with naturality,
$\Sigma^{\alpha,s}\in\EqvCont(Z,Z)$.
\end{proof}

\begin{lemma}[Category convolution is equivariant]
\label{lem:catconv-natural}
$\widetilde L_{\mathsf K}$ is a continuous natural (equivariant) transformation $Z\Rightarrow Z'$, hence $\widetilde L_{\mathsf K}\in\EqvCont(Z,Z')$.
\end{lemma}

\begin{proof}
\emph{(i) Measurability, well-definedness.}
By contravariance (Definition~\ref{feature}),
\[
(Z(u)x_a)(\tau_u y)\;=\;L^Z_u\big(x_a(\pi_u\tau_u y)\big)\in E_Z(s(u)).
\]
Thus the integrand in \eqref{eq:catconv} is
\(
u\mapsto \mathsf K_{s(u)\to a}(u,y)\,(Z(u)x_a)(\tau_u y)
\)
taking values in $E_{Z'}(a)$. Carathéodory regularity of $\mathsf K$ and joint continuity of
$(u,y)\mapsto\tau_u y$ and $(u,x)\mapsto\pi_u x$ imply Borel measurability in $u$. Moreover,
\[
\big\|\mathsf K_{s(u)\to a}(u,y)\,(Z(u)x_a)(\tau_u y)\big\|
 \le \|\mathsf K_{s(u)\to a}(u,y)\|\,\|L^Z_u\|\,\|x_a\|_\infty
 \le G_a(u)\,H_a\,\|x_a\|_\infty,
\]
with $G_a\in L^1(\mu_a)$ from (L$^1$) and $H_a<\infty$ from \eqref{eq:ess-bdd-LZ}. Hence the integrand is Bochner integrable and the definition is meaningful. 

\emph{(ii) Boundedness and continuity $Z(a)\to Z'(a)$.}
Taking $\sup_{y\in\Omega(a)}$ and using the bound above gives
\[
\|(\widetilde L_{\mathsf K}x)_a\|_\infty
\ \le\ \|b_a\|_\infty
\ +\ \|x_a\|_\infty\,H_a\!\int_{I(a)}\!G_a\,d\mu_a,
\]
so $(\widetilde L_{\mathsf K})_a:Z(a)\to Z'(a)$ is bounded linear.
For continuity in $y$, note that for fixed $u$ the map
\[
y\longmapsto \mathsf K_{s(u)\to a}(u,y)\,(Z(u)x_a)(\tau_u y)
=\mathsf K_{s(u)\to a}(u,y)\,L^Z_u\big(x_a(\pi_u\tau_u y)\big)
\]
is continuous by Carathéodory regularity of $\mathsf K$, continuity of
$L^Z_u$, and joint continuity of $\pi,\tau$; dominated convergence with
dominator $G_a(u)\,H_a\,\|x_a\|_\infty$ yields continuity of
$y\mapsto(\widetilde L_{\mathsf K}x)_a(y)$.

\emph{(iii) Naturality (equivariance).}
Let $w:a\to c$ and $x_c\in Z(c)$. For $y\in\Omega(a)$,
\[
\begin{aligned}
\big(Z'(w)(\widetilde L_{\mathsf K}x)_c\big)(y)
&= L^{Z'}_w\Big(b_c(\pi_w y)+\int_{I(c)}
      \mathsf K_{s(u')\to c}(u',\pi_w y)\,(Z(u')x_c)(\tau_{u'}\pi_w y)\,d\mu_c(u')\Big)\\
&= b_a(y)\ +\ \int_{I(a)} \mathsf K_{s(u)\to a}(u,y)\,
       \big(Z(w\!\circ\!u)\,x_c\big)(\tau_u y)\,d\mu_a(u)\\
&= b_a(y)\ +\ \int_{I(a)} \mathsf K_{s(u)\to a}(u,y)\,
       \big(Z(u)\,Z(w)\,x_c\big)(\tau_u y)\,d\mu_a(u)\\
&= \big(\widetilde L_{\mathsf K}(Z(w)x)\big)_a(y),
\end{aligned}
\]
using contravariance $Z(u)\circ Z(w)=Z(w\circ u)$.
Since this holds for all $y$ and $x_c$, we obtain
\( Z'(w)\circ(\widetilde L_{\mathsf K})_c=(\widetilde L_{\mathsf K})_a\circ Z(w)\).

\emph{(iv) Compact--open, finite--object continuity.}
Let $F\subset\Obj\,\C$ be finite and, for each $a\in F$, let $K_a\subset Z(a)$
be compact in the sup--norm. From (ii),
\[
\sup_{x\in K_a}\|(\widetilde L_{\mathsf K}x)_a\|_\infty
\ \le\ \|b_a\|_\infty
\ +\ \Big(H_a\!\int_{I(a)}\!G_a\,d\mu_a\Big)\ \sup_{x\in K_a}\|x\|_\infty,
\]
hence
\begin{align*}
\bigl\|\widetilde L_{\mathsf K}\bigr\|_{(K_a)_{a\in F},\,F}
&:= \max_{a\in F}\ \sup_{x\in K_a}
    \|(\widetilde L_{\mathsf K}x)_a\|_\infty \\
&\le
   \max_{a\in F}\|b_a\|_\infty
   + \max_{a\in F}\!
     \Bigl(H_a\!\int_{I(a)}\!G_a\,d\mu_a\Bigr)
     \max_{a\in F}\ \sup_{x\in K_a}\|x\|_\infty .
\end{align*}
Therefore $\widetilde L_{\mathsf K}\in\EqvCont(Z,Z')$ and is continuous in the
compact--open, finite--object topology.
\end{proof}

\begin{lemma}[CENN is equivariant]\label{lem:closure}
$\CENN_\alpha(X,Y)\subseteq \EqvCont(X,Y)$, and
$\CENN_\alpha(X,Y)$ is closed under composition.
\end{lemma}

\begin{proof}
By Lemma~\ref{lem:catconv-natural}, every category convolution
$\widetilde L_{\mathsf K}:Z\Rightarrow Z'$ (with kernel $\mathsf K$ and natural bias $b$)
is a continuous natural transformation. By Lemma~\ref{prop:gate-general},
every scalar–gated nonlinearity
$\Sigma^{\alpha,s}:Z'\Rightarrow Z'$ (with $s:Z'\Rightarrow S$ a natural scalar channel)
is likewise a continuous natural endomorphism.

By Definition~\ref{def:arrow-bundle}, the arrow–bundle lift
$\Delta_Z:Z\Rightarrow Z_{\downarrow}$ is a continuous natural transformation
(a direct reindexing check). Moreover, 
the admissible componentwise lift $H_{\downarrow}:Z_{\downarrow}\Rightarrow Z'_{\downarrow}$ is also a continuous natural transformation.

Hence each primitive layer admitted in a CENN lies in $\EqvCont(-,-)$ for the
appropriate domain/codomain. Since the composition of continuous natural
transformations is again a continuous natural transformation (whenever
domains/codomains match), any finite composition of such layers belongs to
$\EqvCont(X,Y)$. By definition, a CENN is exactly such a finite composition, so
$\CENN_\alpha(X,Y)\subseteq \EqvCont(X,Y)$ and this class is closed under
composition.
\end{proof}

\begin{lemma}[Local--unit (LU)]
\label{lem:LU_from_regular}
Fix an object $a$ and let $K\subset \Omega(a)$ be compact. Let $\theta$ denote
either the base pullback $\tau$ or a probe family $\sigma$. Then
\[
\lim_{u\to \mathrm{id}_a\atop u\in\mathrm{Hom}_{\mathcal C}(a,a)}
\ \sup_{y\in K}\,\bigl\|\pi_u(\theta_u y)-y\bigr\| \;=\;0.
\]
\end{lemma}

\begin{proof}
Define 
\[
F_\theta:\mathrm{Hom}_{\mathcal C}(a,a)\times\Omega(a)\to\mathbb{R}_{\ge 0},\qquad
F_\theta(u,y):=\bigl\|\pi_u(\theta_u y)-y\bigr\|.
\]
By continuity of $(u,y)\mapsto \theta_u(y)$ and $(u,x)\mapsto \pi_u(x)$, the composite
$(u,y)\mapsto \pi_u(\theta_u y)$ is continuous; hence so is $F_\theta$.
Moreover, $\theta_{\mathrm{id}_a}=\mathrm{id}$ and $\pi_{\mathrm{id}_a}=\mathrm{id}$ give 
$F_\theta(\mathrm{id}_a,y)=0$ for every $y\in\Omega(a)$.

Fix $\varepsilon>0$. For each $y\in K$, continuity of $F_\theta$ at $(\mathrm{id}_a,y)$ yields open
neighbourhoods $U_y\ni\mathrm{id}_a$ and $V_y\ni y$ with $F_\theta(u,y')<\varepsilon$ for all 
$u\in U_y$ and $y'\in V_y$. Compactness of $K$ gives $y_1,\dots,y_N$ with 
$K\subset\bigcup_{i=1}^N V_{y_i}$. Put $U:=\bigcap_{i=1}^N U_{y_i}$. Then for any $u\in U$ and 
$y\in K$ there exists $i$ with $y\in V_{y_i}$, hence $F_\theta(u,y)<\varepsilon$. Equivalently,
\[
\sup_{y\in K}\bigl\|\pi_u(\theta_u y)-y\bigr\|\ <\ \varepsilon\qquad(u\in U).
\]
Since $\varepsilon>0$ was arbitrary, the limit follows.
\end{proof}

\begin{lemma}[Local boundedness (LB)]
\label{lem:LC_from_regular}
For every compact $K\subset I(a)$,
\[
\sup_{u\in K}\|L^X_u\|<\infty
\quad\text{and}\quad
\sup_{u\in K}\|L^Y_u\|<\infty .
\]
\end{lemma}

\begin{proof}
We first prove that a compact set meets only finitely many components.
Since $I(a)$ is a topological disjoint union, each $\mathrm{Hom}_{\mathcal C}(d,a)$ is open in $I(a)$. 
Suppose $K\subset I(a)$ is compact and meets infinitely many components. 
Choose $u_d\in K\cap \mathrm{Hom}_{\mathcal C}(d,a)$ for infinitely many $d$. Then the sets 
$U_d:=K\cap \mathrm{Hom}_{\mathcal C}(d,a)$ are pairwise disjoint, nonempty, and open in $K$, so 
$\{U_d\}$ is an open cover of $K$ with no finite subcover, which is a contradiction. Hence 
$K$ meets only finitely many components.

We then prove that continuity implies boundedness on each compact component.
Write the finite decomposition 
$K=\bigcup_{i=1}^N K_i$, where $K_i:=K\cap \mathrm{Hom}_{\mathcal C}(d_i,a)$ is compact in 
$\mathrm{Hom}_{\mathcal C}(d_i,a)$. On each $\mathrm{Hom}_{\mathcal C}(d_i,a)$ the maps 
$u\mapsto L^X_u$ and $u\mapsto L^Y_u$ are continuous into a finite-dimensional 
operator space, hence the real–valued functions $u\mapsto \|L^X_u\|$ and 
$u\mapsto \|L^Y_u\|$ are continuous. By the extreme value theorem,
\[
M_i^X:=\sup_{u\in K_i}\|L^X_u\|<\infty,\qquad 
M_i^Y:=\sup_{u\in K_i}\|L^Y_u\|<\infty .
\]
Therefore 
\[
\sup_{u\in K}\|L^X_u\|=\max_{1\le i\le N} M_i^X<\infty,
\qquad
\sup_{u\in K}\|L^Y_u\|=\max_{1\le i\le N} M_i^Y<\infty,
\]
which is the desired (LB) bound.
\end{proof}

\subsection{Proof Step 0}

Fix a finite set $F\subset\Obj\,\C$ and compact $K_a\subset X(a)$ for $a\in F$.


\begin{lemma}[Reduction to finitely many scalar coordinates]\label{lem:step0}
Let $\Phi\in \EqvCont(X,Y)$.
Choose a basis $\{v_{a,i}\}_{i=1}^{q_a}$ of $E_Y(a)$ and
let $\{\lambda_{a,i}\}_{i=1}^{q_a}\subset E_Y(a)^{\!*}$ be the dual basis.
Then there exist constants $C_a\ge 1$ such that, for any family of scalar maps
$g_{a,i}:X(a)\times\Omega(a)\to\mathbb{R}$,
\begin{multline}
\sup_{x_a\in K_a}
  \bigl\|\Phi_a(x_a)-G_a(x_a)\bigr\|_\infty
\le
C_a\,\max_{1\le i\le q_a}
   \sup_{(x_a,y)\in K_a\times\Omega(a)}\\
   \bigl|
     \lambda_{a,i} \bigl(\Phi_a(x_a)(y)\bigr)
     - g_{a,i}(x_a,y)
   \bigr|.
\end{multline}
where $G_a:X(a)\to Y(a)$ is the objectwise assembly
\[
(G_a x_a)(y):=\sum_{i=1}^{q_a} v_{a,i}\,g_{a,i}(x_a,y).
\]
In particular, to approximate $\Phi$ uniformly on $K_a\times\Omega(a)$ it suffices
to approximate the finitely many scalar coordinates
$\,f_{a,i}(x_a,y):=\lambda_{a,i}(\Phi_a(x_a)(y))$.
\end{lemma}

\begin{proof}
Fix $a\in F$ and let $\|\cdot\|_{a}$ be any norm on $E_Y(a)$.
Since $E_Y(a)$ is finite–dimensional, there exist constants $c_a,C_a>0$
(depending only on the chosen bases) such that for all $w\in E_Y(a)$,
\[
c_a\,\max_{1\le i\le q_a} |\lambda_{a,i}(w)|
\ \le\ \|w\|_{a}
\ \le\ C_a\,\max_{1\le i\le q_a} |\lambda_{a,i}(w)|.
\]
For $x_a\in X(a)$ and $y\in\Omega(a)$ put
$\Delta_a(x_a,y):=\Phi_a(x_a)(y)-G_a(x_a)(y)\in E_Y(a)$.
Then
\[
\|\Delta_a(x_a,y)\|_{a}
\ \le\ C_a\,\max_{i} \big|\lambda_{a,i}(\Delta_a(x_a,y))\big|
\ =\ C_a\,\max_{i}\big|f_{a,i}(x_a,y)-g_{a,i}(x_a,y)\big|.
\]
Taking $\sup_{y\in\Omega(a)}$ and then $\sup_{x_a\in K_a}$ yields the claim.
\end{proof}

\begin{corollary}[Uniform scalar reduction]\label{cor:step0}
If for each $a\in F$ and each $1\le i\le q_a$ there are continuous scalars
$g_{a,i}$ with
\[
\sup_{(x_a,y)\in K_a\times\Omega(a)}
\big|f_{a,i}(x_a,y)-g_{a,i}(x_a,y)\big|<\varepsilon_a/C_a,
\]
then the assembled $G=\{G_a\}_{a\in F}$ satisfies
$\sup_{x_a\in K_a}\|\Phi_a(x_a)-G_a(x_a)\|_\infty<\varepsilon_a$.
\end{corollary}

\subsection{Proof Step 1}

\begin{lemma}[Carrier algebra is dense]\label{lem:step1}
For arrows $u\in I(a)$, covectors $\ell\in E_X(s(u))^{*}$ and
weights $\eta\in C(\Omega(a),\mathbb{R})$, define the \emph{carrier}
\[
\phi_{u,\ell,\eta}:K_a\times\Omega(a)\to\mathbb{R},\qquad
\phi_{u,\ell,\eta}(x_a,y):=\eta(y)\,\big\langle \ell,\,(X(u)x_a)\big(\sigma_u y\big)\big\rangle.
\]
Let $\pi_2:K_a\times\Omega(a)\to\Omega(a)$ be the projection and define the
\emph{base weights} $\psi\circ\pi_2$ with $\psi\in C(\Omega(a),\mathbb{R})$.
Let $A_a$ be the unital real subalgebra of $C(K_a\times\Omega(a),\mathbb{R})$
generated by all carriers and base weights.
Then $A_a$ is uniformly dense in $C(K_a\times\Omega(a),\mathbb{R})$.
\end{lemma}

\begin{proof}
\emph{(1) Continuity of generators.}
Fix $u\in I(a)$, $\ell\in E_X(s(u))^{*}$, $\eta\in C(\Omega(a))$.
The evaluation map $\mathrm{ev}:C(\Omega(a),E_X(a))\times\Omega(a)\to E_X(a)$,
$\mathrm{ev}(x,y)=x(y)$, is continuous (the compact--open topology equals the sup--norm topology
since $\Omega(a)$ is compact).
With $u$ fixed, the map $x\mapsto L^X_u\circ x\circ\pi_u$ is continuous, and $y\mapsto \sigma_u y$ is continuous.
Therefore the composite
\[
(x_a,y)\ \mapsto\ x_a\big(\pi_u(\sigma_u y)\big)\ \mapsto\
L^X_u\!\big(x_a(\pi_u\sigma_u y)\big)\ \mapsto\
\big\langle \ell,\cdot\big\rangle\ \mapsto\ \eta(y)\,\big\langle \ell,\cdot\big\rangle
\]
is continuous; hence $\phi_{u,\ell,\eta}\in C(K_a\times\Omega(a),\mathbb{R})$.
Likewise $\psi\circ\pi_2$ is continuous. Thus $A_a$ is a unital real subalgebra
of $C(K_a\times\Omega(a),\mathbb{R})$.

\emph{(2) Separation of points.}
Let $(x_a,y)\neq (x'_a,y')$.
If $y\neq y'$, Urysohn’s lemma on the compact Hausdorff space $\Omega(a)$
gives $\psi\in C(\Omega(a),\mathbb{R})$ with $\psi(y)\neq \psi(y')$; then
$\psi\circ\pi_2\in A_a$ separates $(x_a,y)$ and $(x'_a,y')$.

If $y=y'$ and $x_a\neq x'_a$, apply the arrow–evaluation separation at the basepoint $y$:
there exists $u\in I(a)$ such that $v:=(X(u)x_a)(\sigma_u y)\neq v':=(X(u)x'_a)(\sigma_u y)$ in the finite--dimensional space $E_X(s(u))$. Choose $\ell\in E_X(s(u))^{*}$ with $\langle \ell,v\rangle\neq \langle \ell,v'\rangle$ and take $\eta= 1$. Then $\phi_{u,\ell,1}(x_a,y)\neq \phi_{u,\ell,1}(x'_a,y')$, so $A_a$ separates points.

\emph{(3) Stone--Weierstrass.}
By (1) $A_a$ is a unital real subalgebra of $C(K_a\times\Omega(a),\mathbb{R})$, and by (2)
it separates points. The (real) Stone--Weierstrass theorem then implies
$\overline{A_a}^{\|\cdot\|_\infty}=C(K_a\times\Omega(a),\mathbb{R})$.
\end{proof}

\subsection{Proof Step 2}

In the following, we write $I_d(a)=\Hom_\C(d,a)$.

\begin{lemma}[Carrier approximation]\label{lem:step2-final}
Fix $\ell\in E_X(d)^{\!*}$ and a compact $K_a\subset X(a)$.
Define the carrier
\[
c_{u_0,\ell}(x_a,y):=\big\langle \ell,\,(X(u_0)x_a)\big(\sigma_{u_0}y\big)\big\rangle,
\qquad (x_a,y)\in K_a\times\Omega(a).
\]
Then for every $\delta>0$ there exist $\varepsilon>0$ and a map
$h_a:X(a)\to C(\Omega(a),\mathbb{R})$ of the scalar–convolution form
\begin{equation}\label{eq:Ha-final}
(h_a x_a)(y)
\;:=\;
\int_{I_d(a)} \kappa^{(\varepsilon)}_{a,u_0}(u)\,
\Big\langle \ell,\,(X(u)x_a)\big(\sigma_u y\big)\Big\rangle\,d\mu_{d,a}(u),
\end{equation}
such that
\[
\sup_{(x_a,y)\in K_a\times\Omega(a)}
\big|(h_a x_a)(y)-c_{u_0,\ell}(x_a,y)\big|<\delta.
\]
\end{lemma}

\begin{proof}
\emph{Realization.}
Define the arrow–bundle kernel
\[
\mathsf K(u,y)\;:=\;\kappa^{(\varepsilon)}_{a,u_0}(u)\,\ell^{\flat}(u)\;\in\;
L \big(E_X(s(u)),\mathbb{R}\big),
\]
with zero bias, where $\ell^{\flat}(u)$ denotes $\ell\in E_X(d)^*$ viewed in
$L(E_X(s(u)),\mathbb{R})$; this is well‑typed because $u\in I_d(a)$ implies $s(u)=d$.
Then the objectwise operator in \eqref{eq:Ha-final} is exactly the $a$‑component of
the $\sigma$–probe arrow–bundle expression:
\[
\big(L^{\downarrow,\sigma}_{\mathsf K}\circ \Delta_X\big)_a(x_a)(y)
=\int_{I_d(a)} \kappa^{(\varepsilon)}_{a,u_0}(u)\,
\big\langle \ell,(X(u)x_a)(\sigma_u y)\big\rangle\,d\mu_{d,a}(u)
=(h_ax_a)(y).
\]

\emph{Continuity and integrability.}
For fixed $x_a\in X(a)$ define
\[
F_{x_a}(u,y):=\Big\langle \ell,\,(X(u)x_a)\big(\sigma_u y\big)\Big\rangle,
\qquad (u,y)\in I_d(a)\times\Omega(a).
\]
By continuity of $(u,y)\mapsto\sigma_u y$ and $x\mapsto L^X_u\circ x\circ\pi_u$,
the map $(u,y)\mapsto F_{x_a}(u,y)$ is continuous. Let $U\ni u_0$ be a
neighbourhood with compact closure $\overline U\subset I_d(a)$. By local
boundedness of transports, $M_U:=\sup_{u\in\overline U}\|L^X_u\|<\infty$,
whence
\[
|F_{x_a}(u,y)|\le \|\ell\|\cdot \|L^X_u\|\cdot \|x_a\|_\infty
\le \|\ell\|\,M_U\,\sup_{z\in K_a}\|z\|_\infty=:C_U
\]
for all $(u,y)\in\overline U\times\Omega(a)$ and all $x_a\in K_a$.
Choose $\varepsilon<\varepsilon_U$ from (SI) so that
$\operatorname{supp}\kappa^{(\varepsilon)}_{a,u_0}\subset U$.
Then the integrand in \eqref{eq:Ha-final} is bounded by $C_U\,\kappa^{(\varepsilon)}_{a,u_0}(u)$
and continuous in $y$, so \eqref{eq:Ha-final} defines $(h_a x_a)\in C(\Omega(a),\mathbb{R})$.

\emph{Uniform approximation on $K_a\times\Omega(a)$.}
On the compact set $\overline U\times\Omega(a)\times K_a$, the map
$(u,y,x)\mapsto F_x(u,y)$ is continuous; hence $\{F_{x_a}:x_a\in K_a\}$ is
equicontinuous and uniformly bounded there. Given $\eta>0$, choose a
neighbourhood $V\subset U$ of $u_0$ such that
\[
\omega(V):=\sup_{x_a\in K_a}\ \sup_{y\in\Omega(a)}\ \sup_{u\in V}
\big|F_{x_a}(u,y)-F_{x_a}(u_0,y)\big|<\eta.
\]
Refining $\varepsilon$, use (SI) to ensure
$\operatorname{supp}\kappa^{(\varepsilon)}_{a,u_0}\subset V$ and
$\int\kappa^{(\varepsilon)}_{a,u_0}\,d\mu_{d,a}=1$.
Then for all $(x_a,y)\in K_a\times\Omega(a)$,
\[
\big|(h_a x_a)(y)-F_{x_a}(u_0,y)\big|
\le \int_V \kappa^{(\varepsilon)}_{a,u_0}(u)\,\omega(V)\,d\mu_{d,a}(u)
\le \omega(V) < \eta.
\]
Since $\eta>0$ is arbitrary and $F_{x_a}(u_0,y)=c_{u_0,\ell}(x_a,y)$, the desired
uniform bound follows.
\end{proof}

\subsection{Proof Step 3}

For $m\in\mathbb{N}$, define $S^m(a):=C(\Omega(a),\mathbb{R}^m)$ with the sup–norm.
For an arrow $u:b\to a$, the contravariant action on scalars is
$S(u)(r_a)=r_a\circ\pi_u:\Omega(b)\to\mathbb{R}$, and on vectors it is
$S^m(u)(r)=r\circ\pi_u:\Omega(b)\to\mathbb{R}^m$ (componentwise).

Step~2 has produced $m$ scalar channels $h_j:X(a)\to S(a)$, $1\le j\le m$, uniformly defined on the compact $K_a\subset X(a)$. Stack them as
\[
h: X(a)\longrightarrow S^m(a),\qquad
h(x_a)(y)=(h_1(x_a)(y),\dots,h_m(x_a)(y)).
\]
By continuity of each $h_j$ and compactness of $K_a\times\Omega(a)$, the range
\[
\mathcal R_a:=\big\{h(x_a)(y): (x_a,y)\in K_a\times\Omega(a)\big\}\ \subset\ \mathbb{R}^m
\]
is compact. Equivalently, the map $(x,y)\mapsto h_j(x)(y)$ is continuous on $K_a\times\Omega(a)$, hence uniformly continuous in $y$ uniformly over $x\in K_a$.

We define identity–supported scalar operators. Fix $a\in\mathrm{Ob}\,C$ and let $\kappa^{(\varepsilon)}_{a,\mathrm{id}_a}$ be the approximate identity from (SI) on $I(a)$ supported in a small neighbourhood of $\mathrm{id}_a$. For a continuous field $A:\Omega(a)\to L(\mathbb{R}^m,\mathbb{R}^q)$, define the (objectwise) operator
\[
\big(\mathcal L^{(\varepsilon)}_{A} r\big)(y)
  \;:=\; b(y)\;+\;\int_{I(a)} \kappa^{(\varepsilon)}_{a,\mathrm{id}_a}(u)\, A(y)\!\left(\big(S^m(u)r\big)(\sigma_u y)\right) d\mu_a(u) 
\ \ \ (r\in S^m(a))
\]
for any bias $b\in S^q(a)$. Since $(S^m(u)r)(\sigma_u y)=r(\pi_u\sigma_u y)$, this is well‑typed and continuous in $y$ by Carathéodory and dominated convergence.\footnote{In the discrete track with counting measure, one may take $\kappa^{(\varepsilon)}_{a,\mathrm{id}_a}=\mathbf{1}_{\{\mathrm{id}_a\}}$ 
(which denotes the indicator function).}

\begin{lemma}[Affine closure]\label{lem:step3-SS}
Let $h_1,\dots,h_m:X(a)\to S(a)$ be continuous scalar channels, stack $h:X(a)\to S^m(a)$, and fix coefficients $c\in\mathbb{R}^m$ and a base weight $b\in S(a)$. 
Then for every $\varepsilon>0$ there exists an identity–supported scalar operator $\widetilde{\mathcal L}^{(\varepsilon)}:S^m(a)\to S(a)$ of the above form such that, for all $x_a\in K_a$,
\[
\sup_{y\in\Omega(a)}
\Big|\big(\widetilde{\mathcal L}^{(\varepsilon)}\! \circ h(x_a)\big)(y)
-\Big(\sum_{j=1}^m c_j\,h_j(x_a)(y)\,+\,b(y)\Big)\Big|<\varepsilon.
\]
\end{lemma}

\begin{proof}
Write $r:=h(x_a)\in S^m(a)$. 
Instantiate the construction with $q=1$ and $A(y)z:=\langle c,z\rangle$; let $\widetilde{\mathcal L}^{(\varepsilon)}:=\mathcal L^{(\varepsilon)}_{A}$ and keep bias $b_a=b$. Then\footnote{For $r\in S^m(a)$, we write
$r(y)=(r_1(y),\dots,r_m(y))$ with 
$r_j\in S(a)=C(\Omega(a),\mathbb{R})$ 
denoting the $j$-th scalar component.}
\begin{multline}
(\widetilde{\mathcal L}^{(\varepsilon)} r)(y)
= b(y)
+ \int_{I(a)} \kappa^{(\varepsilon)}_{a,\mathrm{id}_a}(u)\,
  \big\langle c,\,(S^m(u)r)(\sigma_u y)\big\rangle\,\mathrm{d}\mu_a(u) \\
= b(y)
+ \int_{I(a)} \kappa^{(\varepsilon)}_{a,\mathrm{id}_a}(u)\,
  \sum_{j=1}^m c_j\, r_j(\pi_u\sigma_u y)\,\mathrm{d}\mu_a(u).
\end{multline}
Fix $\delta>0$. By (LU), there exists a neighbourhood $U\ni\mathrm{id}_a$ in $I(a)$ such that
\[
\sup_{y\in\Omega(a)}\|\pi_u\sigma_u y - y\|\ <\ \delta\qquad\text{for all }u\in U.
\]
Since $(x,y)\mapsto h_j(x)(y)$ is uniformly continuous on $K_a\times\Omega(a)$, there exists $\omega_j(\delta)\to 0$ as $\delta\to 0$ with
\[
\sup_{x\in K_a}\sup_{y\in\Omega(a)}\ \big|\,h_j(x)(\pi_u\sigma_u y) - h_j(x)(y)\,\big|\ \le\ \omega_j(\delta)\qquad (u\in U).
\]
Choose $\varepsilon$ small enough that $\mathrm{supp}\,\kappa^{(\varepsilon)}_{a,\mathrm{id}_a}\subset U$. Then, uniformly in $x_a\in K_a$ and $y\in\Omega(a)$,
\[
\Big|\int \kappa^{(\varepsilon)}_{a,\mathrm{id}_a}(u)\,\sum_{j=1}^m c_j\, r_j(\pi_u\sigma_u y)\,d\mu_a(u)
- \sum_{j=1}^m c_j\, r_j(y)\Big|
\ \le\ \sum_{j=1}^m |c_j|\,\omega_j(\delta).
\]
Let $\delta$ (hence $\varepsilon$) make the right–hand side $<\varepsilon$. This yields the claim.
\end{proof}


\subsection{Proof Step 4}

\begin{lemma}[Pointwise scalar MLP via gates]\label{lem:step4-MLP}
Let $\alpha:\mathbb{R}\to\mathbb{R}$ be nonpolynomial and continuous. Let
$h=(h_1,\dots,h_m):X(a)\to S^m(a)$ be continuous, and let
$F:\mathbb{R}^m\to\mathbb{R}$ be continuous. Then for every $\varepsilon>0$ there exists
a finite composition of layers of two kinds:
\begin{enumerate}
\item identity–supported scalar operators $S^p(a)\to S^q(a)$;
\item generalized gates $\Sigma^{\alpha,s}:S^q(a)\to S^q(a)$
      with $s$ a continuous scalar channel 
\end{enumerate}
such that the resulting map $\mathsf{MLP}_h:X(a)\to S(a)$ satisfies
\[
\sup_{(x_a,y)\in K_a\times\Omega(a)}
\big|\mathsf{MLP}_h(x_a)(y)-F\big(h(x_a)(y)\big)\big|<\varepsilon.
\]
\end{lemma}

\begin{proof}
Since $h$ is continuous and $K_a\times\Omega(a)$ compact, the range
$\mathcal R_a:=h(K_a\times\Omega(a))\subset\mathbb{R}^m$ is compact.
By the classical UAT (Leshno–Lin–Pinkus–Schocken theorem), there exists a finite‑depth
MLP $\mathcal N:\mathbb{R}^m\to\mathbb{R}$ built from affine maps and $\alpha$ with
$\sup_{z\in\mathcal R_a}|\mathcal N(z)-F(z)|<\varepsilon/2$.

We now simulate the MLP pointwise in $y$. 
We first consider affine layers. Each affine \(z\mapsto Az+b\) is implemented as in
Lemma~\ref{lem:step3-SS} (with \(m\) replaced by the current width),
giving \(\sup_{(x,y)}\)‑error \(<\delta\) per affine, for any preset \(\delta>0\).
We next consider nonlinearity \(\alpha\). 
For a \(q\)–channel signal \(r\in S^q(a)\), to produce \(y\mapsto \alpha(r_j(y))\), use:
\[
T_j(r)(y)=(1,\ r_j(y))\in\mathbb{R}^2,\quad
(\Sigma^{\alpha,s_j}\circ T_j)(r)(y)=\big(\alpha(r_j(y)),\,\alpha(r_j(y))\,r_j(y)\big),
\]
where \(s_j(r)=r_j\) is the natural coordinate projection. Project the
first coordinate by an identity–supported affine map. Doing this in parallel
for \(j=1,\dots,q\) yields the coordinatewise map \(r\mapsto \alpha(r)\).

Compose the simulated affine and nonlinear layers in the order of
\(\mathcal N\), starting from the \(m\)–channel input \(h(x_a)\).
By construction, the only approximation comes from the (approximate)
affine realizations; choose their tolerances so that the total accumulated
error is \(<\varepsilon/2\) on \(K_a\times\Omega(a)\) (standard stability of
pointwise composition under sup–norm). Then
\begin{align*}
\sup_{(x_a,y)}
  \bigl|\mathsf{MLP}_h(x_a)(y)-F(h(x_a)(y))\bigr|
&\le
  \underbrace{\sup
    \bigl|\mathsf{MLP}_h-\mathcal N\!\circ h\bigr|}_{<\,\varepsilon/2} \\
&\quad+\underbrace{\sup
    \bigl|\mathcal N\!\circ h - F\!\circ h\bigr|}_{<\,\varepsilon/2} \\
&< \varepsilon .
\end{align*}
All layers are continuous in the sup–norm; hence $\mathsf{MLP}_h$ is continuous.
\end{proof}

\subsection{Proof Step 5}

Let us fix $\Phi\in \EqvCont(X,Y)$. 

\begin{lemma}[Approximation of the scalar coordinates]\label{lem:step5}
For each $a\in F$, fix a basis $\{v_{a,i}\}_{i=1}^{q_a}$ of $E_Y(a)$ 
and the dual basis $\{\lambda_{a,i}\}_{i=1}^{q_a}\subset E_Y(a)^{\!*}$.
Define the scalar coordinate targets
\[
f_{a,i}:K_a\times\Omega(a)\longrightarrow\mathbb{R},\qquad
f_{a,i}(x_a,y):=\big\langle \lambda_{a,i},\,\Phi_a(x_a)(y)\big\rangle.
\]
Then, for every $\varepsilon>0$ and for each $a\in F$ and $1\le i\le q_a$,
there exists a continuous scalar channel
\[
g_{a,i}: X(a)\longrightarrow S(a),\quad
x_a\mapsto g_{a,i}(x_a)\in C(\Omega(a),\mathbb{R}),
\]
constructed as a finite composition of the Step~2–Step~4 primitives at $a$ such that
\[
\sup_{(x_a,y)\in K_a\times\Omega(a)}
\big|g_{a,i}(x_a)(y)-f_{a,i}(x_a,y)\big|
\;<\;\varepsilon.
\]
\end{lemma}

\begin{proof}
Fix $a\in F$ and $i$.
By Step~1, for every $\eta_1>0$ there exist
a finite list of carriers $\phi_{k}:K_a\times\Omega(a)\to\mathbb{R}$,
\[
\phi_k(x_a,y)=\eta_k(y)\,\big\langle \ell_k,\,(X(u_k)x_a)\big(\sigma_{u_k}y\big)\big\rangle,
\qquad k=1,\dots,m,
\]
with $u_k\in I(a)$, $\ell_k\in E_X(s(u_k))^{\!*}$, $\eta_k\in C(\Omega(a))$, and a polynomial
$P:\mathbb{R}^m\to\mathbb{R}$ such that
\begin{equation}\label{eq:SW-approx}
\sup_{(x_a,y)\in K_a\times\Omega(a)}
\big|\,f_{a,i}(x_a,y) - P\big(\phi_1(x_a,y),\ldots,\phi_m(x_a,y)\big)\,\big|
<\eta_1.
\end{equation}
We let $\Phi^\phi:K_a\times\Omega(a)\to\mathbb{R}^m$ denote the stacked carrier map
$\Phi^\phi(x_a,y)=(\phi_1(x_a,y),\ldots,\phi_m(x_a,y))$, and set
\[
\mathcal R_a^\phi:=\Phi^\phi(K_a\times\Omega(a))\subset\mathbb{R}^m,
\]
which is compact. Fix $\eta_2>0$.
By Step~2, for each $k$
there exists a continuous scalar channel $h_k:X(a)\to S(a)$ such that
\begin{equation}\label{eq:carrier-realization}
\sup_{(x_a,y)\in K_a\times\Omega(a)}
\big|\,h_k(x_a)(y)-\phi_k(x_a,y)\,\big|<\eta_2.
\end{equation}
Stack $h=(h_1,\ldots,h_m):X(a)\to S^m(a)$ and let
\[
\mathcal R_a^{h}:=\big\{h(x_a)(y):(x_a,y)\in K_a\times\Omega(a)\big\}\subset\mathbb{R}^m,
\]
also compact. Because $P$ is continuous on the compact
$\mathcal R_a^\phi\cup \mathcal R_a^{h}$, there exists a modulus of continuity
$\omega_P(\cdot)$ with $\omega_P(\delta)\to 0$ as $\delta\to 0$ such that
\begin{equation}\label{eq:poly-stability}
\sup_{(x_a,y)\in K_a\times\Omega(a)}
\Big|\,P\big(h(x_a)(y)\big)-P\big(\Phi^\phi(x_a,y)\big)\,\Big|
\le \omega_P\!\left(\sup_k\|h_k-\phi_k\|_\infty\right)
\ \le\ \omega_P(\eta_2).
\end{equation}
Now apply Step~4 to $F:=P$ and the input $h$:
for every $\eta_3>0$ there exists a scalar CENN block
$\mathsf{MLP}_h:X(a)\to S(a)$ such that
\begin{equation}\label{eq:MLP-approx}
\sup_{(x_a,y)\in K_a\times\Omega(a)}
\Big|\,\mathsf{MLP}_h(x_a)(y)-P\big(h(x_a)(y)\big)\,\Big|
<\eta_3.
\end{equation}
Combine \eqref{eq:SW-approx}–\eqref{eq:MLP-approx}:
\begin{align}
\sup_{(x_a,y)}
  \bigl|\mathsf{MLP}_h(x_a)(y)-f_{a,i}(x_a,y)\bigr|
&\le \underbrace{\sup_{(x_a,y)}
       \bigl|\mathsf{MLP}_h(x_a)(y)-P(h(x_a)(y))\bigr|}_{<\,\eta_3} \\
&\quad+ \underbrace{\sup_{(x_a,y)}
       \bigl|P(h(x_a)(y))-P(\Phi^\phi(x_a,y))\bigr|}_{\le\,\omega_P(\eta_2)} \\
&\quad+ \underbrace{\sup_{(x_a,y)}
       \bigl|P(\Phi^\phi)-f_{a,i}\bigr|}_{<\,\eta_1}.
\end{align}
Given $\varepsilon>0$, choose $\eta_1,\eta_2,\eta_3>0$ so that
$\eta_1+\omega_P(\eta_2)+\eta_3<\varepsilon$, and set $g_{a,i}:=\mathsf{MLP}_h$.
By construction $g_{a,i}$ is a finite composition of Step~2–Step~4 primitives
(at $a$), hence continuous $X(a)\to S(a)$.
\end{proof}


We also prepare another lemma for the following proofs.

\begin{lemma}[Lipschitz bounds]\label{lem:layer-lip-adm}
For each object $a$, the primitives used in steps~2--4 are globally Lipschitz with constants depending only on the layer parameters at~$a$ as follows:
\begin{enumerate}
\item \textit{Category convolution.} For a category kernel $K$ and natural bias $b$, the linear map
\[
(\widetilde L_{\mathsf K})_a:Z(a)\longrightarrow Z'(a)
\]
is bounded and Lipschitz with
\[
\mathrm{Lip}\big((\widetilde L_{\mathsf K})_a\big)\;\le\;H_a\;\int_{I(a)} G_a\,d\mu_a.
\] 
\item \textit{Arrow--bundle convolution.} For $x=(x_u)_{u\in I(a)}\in Z_\downarrow(a)$,
\[
\mathrm{Lip}\big((L_{\mathsf K}^\downarrow)_a:Z_\downarrow(a)\to Z'(a)\big)\;\le\;\int_{I(a)} G_a\,d\mu_a.
\]
\item \textit{Identity--supported affine scalar operators.} 
For a continuous $A:\Omega(a)\to L(\mathbb{R}^m,\mathbb{R}^q)$ and the objectwise operator $\mathcal L_A^{(\varepsilon)}:S^m(a)\to S^q(a)$ obtained from the identity–supported kernel, one has
\[
\mathrm{Lip}\big(\mathcal L_A^{(\varepsilon)}\big)\;\le\;\sup_{y\in\Omega(a)}\|A(y)\|.
\]
\item \textit{Pointwise activation realized via gate (Step~4).}
For each $j$, the composite
\[
S^q(a)\xrightarrow{\,T_j\,}S^2(a)\xrightarrow{\,\Sigma^{\alpha,s_2}\,}S^2(a)\xrightarrow{\,P_1\,}S(a),
\qquad r\longmapsto \alpha\big(r_j\big),
\]
is globally Lipschitz with constant $\mathrm{Lip}(\alpha)$ where: $P_1$ is the first projection;  $T_{j} : S^{q}(a) \to S^{2}(a)$ is defined by $(T_{j} r)(y) = \big(1,\; r_{j}(y)\big)$; and $s_2 : S^2(a)\longrightarrow S(a)$ is defined by $(s_2(u,v))(y) = v(y)$. The coordinatewise map $r\mapsto \alpha(r)$ on $S^q(a)$ is thus globally Lipschitz with the same constant.
\end{enumerate}
Finite sums and compositions of the maps in 1--4 remain globally Lipschitz, with the usual bounds 
$\mathrm{Lip}(f+g)\le \mathrm{Lip}(f)+\mathrm{Lip}(g)$ and 
$\mathrm{Lip}(g\circ f)\le \mathrm{Lip}(g)\,\mathrm{Lip}(f)$.
Consequently, the objectwise block $G_b$ formed at each object $b$ by a finite composition of primitives from 1--4 is globally Lipschitz.
In particular, if we choose the family $G=\{G_b\}$ so that it is arrow–bundle uniformly Lipschitz in the sense that for every object $a$,
\[
L_a\;:=\;\operatorname*{ess\,sup}_{u\in I(a)}\mathrm{Lip}\big(G_{s(u)}\big)\;<\infty,
\]
then the componentwise lift $G_\downarrow:Z_\downarrow\Rightarrow Z'_\downarrow$ is an admissible componentwise lift, hence a continuous natural transformation.
\end{lemma}

\begin{proof}
1. For $x,x'\in Z(a)$,
\[
\begin{aligned}
\bigl\|(\widetilde L_{\mathsf K}x)_a - (\widetilde L_{\mathsf K}x')_a\bigr\|_\infty
&\le \int_{I(a)} \|\mathsf K(u,\cdot)\|\,\|Z(u)(x-x')\|_\infty\,d\mu_a(u)\\
&\le H_a\left(\int G_a\,d\mu_a\right)\|x-x'\|_\infty .
\end{aligned}
\]
using (L1) for $\|K(u,\cdot)\|$ and the above equation \eqref{eq:ess-bdd-LZ} for $\|Z(u)\|$.  Thus $(\widetilde L_{\mathsf K})_a$ is bounded linear with the stated Lipschitz constant.\footnote{These are exactly the Bochner–dominated estimates used to prove continuity of category convolution.}

2. For $x=(x_u)$ and $x'=(x'_u)$ in $Z_\downarrow(a)$,
\[
\begin{aligned}
\bigl\|(L^\downarrow_K x)_a - (L^\downarrow_K x')_a\bigr\|_\infty
&\le \int_{I(a)} \|K(u,\cdot)\|\,\|x_u - x'_u\|_\infty\,d\mu_a(u)\\
&\le \left(\int G_a\,d\mu_a\right)\,\|x - x'\|_{Z_\downarrow(a)} .
\end{aligned}
\]
because $\|x-x'\|_{Z_\downarrow(a)}=\operatorname{ess\,sup}_u\|x_u-x'_u\|_\infty$ and $\|K(u,\cdot)\|\le G_a(u)$ a.e. Hence $(L^\downarrow_K)_a$ is Lipschitz with the stated constant.

3. For $r,r'\in S^m(a)$ and $y\in\Omega(a)$, the identity–supported formula gives
\[
\begin{aligned}
\bigl\|\big(\mathcal L_A^{(\varepsilon)}r\big)(y)
      - \big(\mathcal L_A^{(\varepsilon)}r'\big)(y)\bigr\|
&\le \int_{I(a)} \kappa^{(\varepsilon)}_{a,\mathrm{id}}(u)\,\|A(y)\| \\
&\qquad\cdot
   \bigl\|(S^m(u)r)(\sigma_u y)
        - (S^m(u)r')(\sigma_u y)\bigr\|\,d\mu_a(u) .
\end{aligned}
\]
and $\|S^m(u)\|=1$ for the sup norms (precomposition cannot increase the sup), while $\int\kappa^{(\varepsilon)}_{a,\mathrm{id}}\,d\mu_a=1$.  Taking the sup over $y$ yields
\[
\|\mathcal L_A^{(\varepsilon)}r-\mathcal L_A^{(\varepsilon)}r'\|_\infty\le \Big(\sup_{y\in\Omega(a)}\|A(y)\|\Big)\,\|r-r'\|_\infty.
\]

4. The map $S^q(a)\to S(a)$, $r\mapsto \alpha(r_j)$, is the composition
\[
S^q(a)\xrightarrow{\ \mathrm{pr}_j\ } S(a)\xrightarrow{\ r\mapsto \alpha\circ r\ }S(a),
\]
where $\|\mathrm{pr}_j\|\le 1$ and the Nemytskii operator $r\mapsto \alpha\circ r$ is Lipschitz with constant $\mathrm{Lip}(\alpha)$ in the sup norm (pointwise Lipschitzness of $\alpha$ transfers under $\sup$ on a compact base).  Hence $\mathrm{Lip}(r\mapsto \alpha(r_j))\le \mathrm{Lip}(\alpha)$.  The $T_j\!-\!\Sigma^{\alpha,s_2}\!-\!P_1$ gadget realizes precisely this map; $T_j$ and $P_1$ are linear with operator norms~$\le 1$, so the realized map has Lipschitz constant $\mathrm{Lip}(\alpha)$.

For the closure properties, use the inequalities for sums and compositions quoted in the statement. Since, at each object $b$, the block $G_b$ is a finite composition and sum of the primitives, it is globally Lipschitz with a constant depending only on the layer parameters used at $b$.

Finally, by assumption we choose $G$ so that for every object $a$,
\[
L_a=\operatorname*{ess\,sup}_{u\in I(a)}\mathrm{Lip}\big(G_{s(u)}\big)<\infty,
\]
i.e.\ $G$ is arrow–bundle uniformly Lipschitz. By the definition of admissible componentwise lifts, this is exactly the condition ensuring that the componentwise lift $G_\downarrow:Z_\downarrow\Rightarrow Z'_\downarrow$ is admissible (hence continuous natural transformation).\footnote{This is the form used subsequently when compiling arrow–bundle blocks.} 
\end{proof}

\subsection{Proof Step 6}

Fix finite $F\subset\Obj\,\C$ and compact $K_a\subset X(a)$ for $a\in F$.

Let us recall the equivariant compilation assumption (EC): There exists a continuous natural transformation
\[
R:\;Y_{\downarrow}\Rightarrow Y,\qquad R\circ\Delta_Y=\mathrm{id}_Y,
\]
realized as an arrow–bundle convolution $L^{\downarrow}_{\mathsf R}$ with zero natural bias, whose kernel $\mathsf R_{\,s(u)\to a}(u,y)\in L \big(E_Y(s(u)),E_Y(a)\big)$ satisfies \textup{(C)}, \textup{($\mathrm{IN}_{\downarrow}$)}, \textup{(L1)}. 
Concretely, for $x=(x_u)_{u\in I(a)}\in Y_{\downarrow}(a)$,
\[
\big(L^{\downarrow}_{\mathsf R}x\big)_a(y)\;=\;\int_{I(a)}\mathsf R_{\,s(u)\to a}(u,y)\,x_u(\tau_u y)\,d\mu_a(u).
\]

\begin{lemma}[Arrow–bundle and componentwise lifts]\label{lem:lift-new}
Let $G=\{G_b:X(b)\to Y(b)\}_{b\in\Obj\,\C}$ be arrow–bundle uniformly Lipschitz. 
Then
\[
\Delta_X:X\Rightarrow X_{\downarrow}
\quad\text{and}\quad
G_{\downarrow}:X_{\downarrow}\Rightarrow Y_{\downarrow},\quad
\big(G_{\downarrow}\big)_a\!\big((x_u)_{u\in I(a)}\big)
=\big(G_{s(u)}(x_u)\big)_{u\in I(a)}
\]
are continuous natural transformations.
\end{lemma}

\begin{proof}
\emph{Functoriality and well-definedness.}
By definition, $Z\mapsto Z_{\downarrow}$ is a functor with
$Z_{\downarrow}(w)\big((h_{u'})_{u'\in I(c)}\big)_u:=h_{w\circ u}$.
Null-set preservation (NSP) makes this well-defined. 

\emph{Naturality of $\Delta_X$.}
Let $w:a\to c$ and $x_c\in X(c)$. Then
\[
\big(X_{\downarrow}(w)\circ(\Delta_X)_c(x_c)\big)_u
= X(w\circ u)(x_c)
= \big((\Delta_X)_a\circ X(w)(x_c)\big)_u,
\]
since $s(w\circ u)=s(u)$ and $X$ is contravariant.

\emph{Continuity of $\Delta_X$.}
With essential–sup norm $\|(\cdot)\|_{X_{\downarrow}(a)}:=\operatorname*{ess\,sup}_{u\in I(a)}\|\cdot\|_\infty$,
\[
\|(\Delta_X)_a(x)\|_{X_{\downarrow}(a)}
=\operatorname*{ess\,sup}_{u\in I(a)}\|X(u)x\|_\infty
\le \Big(\operatorname*{ess\,sup}_{u\in I(a)}\|L^X_u\|\Big)\,\|x\|_\infty,
\]
which is finite by the transport bound. Hence $(\Delta_X)_a$ is bounded (thus continuous).

\emph{Naturality of $G_{\downarrow}$.}
For $w:a\to c$ and $(x_{u'})_{u'\in I(c)}\in X_{\downarrow}(c)$,
\[
\Big(Y_{\downarrow}(w)\circ G_{\downarrow,c}\circ X_{\downarrow}(w)\Big)\big((x_{u'})\big)_u
= G_{s(w\circ u)}\!\big(x_{w\circ u}\big)
= \big(G_{\downarrow,a}\circ X_{\downarrow}(w)\big)\big((x_{u'})\big)_u,
\]
since $s(w\circ u)=s(u)$.

\emph{Continuity of $G_{\downarrow}$.}
Fix $a$ and write $L_a:=\operatorname*{ess\,sup}_{u\in I(a)}\mathrm{Lip}(G_{s(u)})<\infty$.
For $x=(x_u)$ and $x'=(x'_u)$ in $X_{\downarrow}(a)$,
\[
\begin{aligned}
\|G_{\downarrow,a}(x)-G_{\downarrow,a}(x')\|_{Y_{\downarrow}(a)}
&= \operatorname*{ess\,sup}_{u\in I(a)}
   \|G_{s(u)}(x_u)-G_{s(u)}(x'_u)\|_\infty \\
&\le L_a\,\operatorname*{ess\,sup}_{u\in I(a)}
        \|x_u-x'_u\|_\infty \\
&= L_a\,\|x-x'\|_{X_{\downarrow}(a)}.
\end{aligned}
\]
Thus $G_{\downarrow,a}$ is Lipschitz, hence continuous. Since the argument holds for each $a$, $G_{\downarrow}$ is a continuous natural transformation.
\end{proof}

\begin{lemma}[Equivariant convolutional retraction]\label{lem:Rlayer-new}
$L^{\downarrow}_{\mathsf R}:Y_{\downarrow}\Rightarrow Y$ is a continuous natural
transformation. Moreover, for each $a$, there exists $G_a^{\mathsf R}\in L^1(\mu_a)$ with
\[
\sup_{y\in\Omega(a)}\ \big\|\mathsf R_{\,s(u)\to a}(u,y)\big\|\ \le\ G_a^{\mathsf R}(u)\quad
\text{for $\mu_a$–a.e.\ $u\in I(a)$}.
\]
\end{lemma}

\begin{proof}
\emph{Continuity.} Fix $a$. By (C), for each fixed $u$ the map $y\mapsto \mathsf R(u,y)$ is continuous; for fixed $y$,
$u\mapsto\mathsf R(u,y)$ is measurable. Given $x=(x_u)_{u\in I(a)}\in Y_{\downarrow}(a)$, joint continuity of
$(u,y)\mapsto \tau_u y$ and continuity of $x_u$ yield that $y\mapsto \mathsf R(u,y)\,x_u(\tau_u y)$ is continuous for each $u$,
and (L1) provides a Bochner–integrable dominator $G_a^{\mathsf R}(u)\sup_{y}\|x_u(\tau_u y)\|\le G_a^{\mathsf R}(u)\|x_u\|_\infty$.
Dominated convergence gives $(L^{\downarrow}_{\mathsf R}x)_a\in C(\Omega(a),E_Y(a))$. A similar argument shows $x\mapsto L^{\downarrow}_{\mathsf R}(x)$
is continuous in the sup–norm.

\emph{Naturality.} Let $w:a\to c$, $f=(f_{u'})_{u'\in I(c)}\in Y_{\downarrow}(c)$ and $y\in\Omega(a)$. Then
\[
\begin{aligned}
\big(Y(w)L^{\downarrow}_{\mathsf R}f\big)_a(y)
&= L^Y_w\!\int_{I(c)}\!\mathsf R(u',\pi_w y)\,f_{u'}(\tau_{u'}\pi_w y)\,d\mu_c(u')\\
&\stackrel{\mathrm{IN}_{\downarrow}}{=}\int_{I(a)}\!\mathsf R(u,y)\,f_{w\circ u}(\tau_u y)\,d\mu_a(u)\\
&= \big(L^{\downarrow}_{\mathsf R}\,Y_{\downarrow}(w)f\big)_a(y),
\end{aligned}
\]
which is exactly $Y(w)\circ L^{\downarrow}_{\mathsf R}=L^{\downarrow}_{\mathsf R}\circ Y_{\downarrow}(w)$. The $L^1$ bound is the stated assumption.
\end{proof}


\begin{lemma}[Projection on natural maps]\label{lem:proj}
If $G\in\EqvCont(X,Y)$ then
\[
L^{\downarrow}_{\mathsf R}\circ G_{\downarrow}\circ\Delta_X \;=\; G.
\]
\end{lemma}

\begin{proof}
By naturality of $G$, we have $G_{\downarrow}\circ\Delta_X=\Delta_Y\circ G$. Apply $L^{\downarrow}_{\mathsf R}$ and use the retraction
property $R\circ\Delta_Y=\mathrm{id}_Y$ from the assumption:
\[
L^{\downarrow}_{\mathsf R}\circ G_{\downarrow}\circ\Delta_X \;=\; L^{\downarrow}_{\mathsf R}\circ\Delta_Y\circ G \;=\; G.
\]
\end{proof}

\begin{lemma}[Localized stability on transported compacts]\label{lem:stab-uniform}
Let $F\subset\Obj\,\C$ be finite and, for each $a\in F$, let $K_a\subset X(a)$ be compact.
For any objectwise continuous families $G=\{G_c:X(c)\to Y(c)\}_{c\in \Obj\,\C}$ and $H=\{H_c:X(c)\to Y(c)\}_{c\in \Obj\,\C}$, one has
\begin{equation}\label{eq:stab-local-master}
\begin{split}
\big\|\,L^{\downarrow}_{\mathsf R}\!\circ G_{\downarrow}\!\circ\!\Delta_X
- L^{\downarrow}_{\mathsf R}\!\circ H_{\downarrow}\!\circ\!\Delta_X\,\big\|_{(K_a),F}
\ \le\ &
\max_{a\in F}\ \int_{I(a)}\! G_a^{\mathsf R}(u)\, \\
&\times
\sup_{x\in X(u)(K_a)} 
  \|G_{s(u)}(x)-H_{s(u)}(x)\|_\infty\, d\mu_a(u).
\end{split}
\end{equation}
In particular, for any measurable subsets $U_a\subset I(a)$,
\begin{equation}\label{eq:stab-split}
\big\|\,L^{\downarrow}_{\mathsf R}\!\circ G_{\downarrow}\!\circ\!\Delta_X
- L^{\downarrow}_{\mathsf R}\!\circ H_{\downarrow}\!\circ\!\Delta_X\,\big\|_{(K_a),F}
\ \le\
\max_{a\in F}\Big\{A_a(G,H;U_a)+T_a(G,H;U_a)\Big\},
\end{equation}
where
\[
A_a(G,H;U_a)\ :=\ \int_{U_a}\!\! G_a^{\mathsf R}(u)\,
\sup_{x\in X(u)(K_a)} \|G_{s(u)}(x)-H_{s(u)}(x)\|_\infty\, d\mu_a(u),
\]
\[
T_a(G,H;U_a)\ :=\ \int_{I(a)\setminus U_a}\!\! G_a^{\mathsf R}(u)\,
\sup_{x\in X(u)(K_a)} \|G_{s(u)}(x)-H_{s(u)}(x)\|_\infty\, d\mu_a(u).
\]
\end{lemma}

\begin{proof}
Fix $a\in F$ and $x_a\in K_a$. Using the convolution form of $L^{\downarrow}_{\mathsf R}$ and the bound
$\sup_{y\in\Omega(a)}\|\mathsf R(u,y)\|\le G_a^{\mathsf R}(u)$ for $\mu_a$–a.e.~$u$,
\[
\begin{aligned}
&\ \ \ \ \big\|\big(L^{\downarrow}_{\mathsf R}\!\circ G_{\downarrow}\!\circ\!\Delta_X
- L^{\downarrow}_{\mathsf R}\!\circ H_{\downarrow}\!\circ\!\Delta_X\big)_a(x_a)\big\|_\infty\\
&\le \int_{I(a)}\! \Big(\sup_{y}\|\mathsf R(u,y)\|\Big)\,
\big\|\big(G_{s(u)}-H_{s(u)}\big)\big(X(u)x_a\big)\big\|_\infty\, d\mu_a(u)\\
&\le \int_{I(a)}\! G_a^{\mathsf R}(u)\,
\sup_{x\in X(u)(K_a)} \|G_{s(u)}(x)-H_{s(u)}(x)\|_\infty\, d\mu_a(u).
\end{aligned}
\]
Taking $\sup_{x_a\in K_a}$ and then $\max_{a\in F}$ yields \eqref{eq:stab-local-master}. The
decomposition \eqref{eq:stab-split} is just splitting the integral over $U_a$ and its complement.
\end{proof}

For each $a\in F$ choose an increasing sequence of compacts $U_{a,n}\nearrow\supp\mu_a$ in $I(a)$.
Since $I(a)=\bigsqcup_{b}\Hom_{\C}(b,a)$ is a topological disjoint union, a compact subset meets only
finitely many components; hence $F^{+,(n)}:=\{\,s(u):u\in U_{a,n},\,a\in F\,\}$ is finite.
For each $b\in F^{+,(n)}$ define the enlarged sets
\[
K^{+,(n)}_b
:=\overline{\bigcup_{a\in F}\ \bigcup_{\substack{u\in U_{a,n}\\ s(u)=b}}
X(u)\big(K_a\big)}\ \subset X(b).
\]

\begin{lemma}[Transport continuity on compacts]\label{lem:transport-cont}
For fixed $a,b$, the map
\[
\Hom_{\C}(b,a)\times X(a)\longrightarrow X(b),\qquad (u,x)\longmapsto X(u)(x)=L^X_u\circ x\circ\pi_u,
\]
is continuous on $U\times K$ for every compact $U\subset \Hom_{\C}(b,a)$ and compact $K\subset X(a)$
(with the sup–norms on $X(\cdot)$).
\end{lemma}

\begin{proof}
By Definition~\ref{feature}, $(u,y)\mapsto \pi_u(y)$ and $u\mapsto L^X_u$ are continuous; on the
compact $U\times\Omega(b)$ they are uniformly continuous and uniformly bounded. The evaluation
map $C(\Omega(a),E_X(a))\times\Omega(a)\to E_X(a)$ is continuous (sup–norm on $C(\Omega(a),\cdot)$).
Thus $(u,x,y)\mapsto L^X_u\big(x(\pi_u(y))\big)$ is continuous on $U\times K\times\Omega(b)$, hence
uniformly continuous; taking $\sup_{y\in\Omega(b)}$ yields continuity into $X(b)$.
\end{proof}

By Lemma~\ref{lem:transport-cont} and compactness of $U_{a,n}\times K_a$, each $K^{+,(n)}_b$ is compact. 

\begin{lemma}[Extension realization]\label{lem:cutoff}
For every $\delta>0$ and $n\in\mathbb N$ there exists an objectwise continuous family
$G^{(n)}=\{G^{(n)}_b:X(b)\to Y(b)\}$ such that:
\begin{enumerate}
\item \textit{Approximation on transported compacts:}\;
For all $b\in F^{+,(n)}$,
\[
\sup_{x\in K^{+,(n)}_b}\ \|G^{(n)}_b(x)-\Phi_b(x)\|_\infty\ <\ \delta.
\]
\item \textit{Lipschitz completion outside $F^{+,(n)}$:}\;
For all $b\notin F^{+,(n)}$, $G^{(n)}_b= 0$.
\item \textit{Arrow–bundle uniformly Lipschitz:}\;
There exists a finite constant
\[
L_*^{(n)}\ :=\ \max_{b\in F^{+,(n)}} \mathrm{Lip}\big(G^{(n)}_b\big)\ <\ \infty
\]
such that, for every object $a$,
\[
\operatorname*{ess\,sup}_{u\in I(a)}\ \mathrm{Lip}\big(G^{(n)}_{s(u)}\big)\ \le\ L_*^{(n)}.
\]
In particular, $G^{(n)}$ is arrow–bundle uniformly Lipschitz, so the componentwise
lift $(G^{(n)})_{\downarrow}:X^{\downarrow}\Rightarrow Y^{\downarrow}$ is an admissible 
componentwise lift.
\end{enumerate}
\end{lemma}

\begin{proof}
Apply Step~5 (objectwise approximation on compacts) to the finite family
$\{K^{+,(n)}_b\}_{b\in F^{+,(n)}}$ with tolerance $\delta/2$ to obtain
objectwise continuous maps $G^{\mathrm{aux}}_b:X(b)\to Y(b)$ $(b\in F^{+,(n)})$ satisfying
\[
\sup_{x\in K^{+,(n)}_b}\ \|G^{\mathrm{aux}}_b(x)-\Phi_b(x)\|_\infty<\delta/2.
\]
Note that each $K^{+,(n)}_b$ is compact by transport continuity on compacts. 

We apply the McShane–Whitney extension.
Fix $b\in F^{+,(n)}$ and set $S_b:=K^{+,(n)}_b\subset X(b)$.
By Step~5, $G^{\mathrm{aux}}_b\!\restriction_{S_b}$ is Lipschitz w.r.t.\ the sup metrics on $X(b),Y(b)$.
Choose a linear isomorphism $E_Y(b)\cong\mathbb R^{q_b}$ and extend coordinatewise:
for $1\le i\le q_b$ and $(x,y)\in X(b)\times\Omega(b)$ define the McShane envelopes
\[
\begin{aligned}
\underline G_{b,i}(x)(y)
&:= \sup_{k\in S_b}
    \bigl\{G^{\mathrm{aux}}_{b,i}(k)(y)
          - L_b\,\|x-k\|_\infty\bigr\},\\
\overline G_{b,i}(x)(y)
&:= \inf_{k\in S_b}
    \bigl\{G^{\mathrm{aux}}_{b,i}(k)(y)
          + L_b\,\|x-k\|_\infty\bigr\}.
\end{aligned}
\]
where $L_b:=\mathrm{Lip}\big(G^{\mathrm{aux}}_b\!\restriction_{S_b}\big)$, and put
$\widehat G^{\mathrm{raw}}_{b,i}:=\frac12(\underline G_{b,i}+\overline G_{b,i})$ and
$\widehat G_b(x)(y):=(\widehat G^{\mathrm{raw}}_{b,1}(x)(y),\ldots,\widehat G^{\mathrm{raw}}_{b,q_b}(x)(y))$.
By the maximum theorem, it follows that
$y\mapsto\underline G_{b,i}(x)(y),\overline G_{b,i}(x)(y)$ are continuous;
hence $\widehat G_b(x)\in Y(b)$ for each $x$, and $x\mapsto\widehat G_b(x)$ is globally
Lipschitz (with constant $\le C_bL_b$ for a norm‑equivalence constant $C_b$ depending only on
the chosen coordinates on $E_Y(b)$). Moreover we have $\widehat G_b\!\restriction_{S_b}=G^{\mathrm{aux}}_b\!\restriction_{S_b}$.

For $b\in F^{+,(n)}$, set $G^{(n)}_b=\widehat G_b$; for $b\notin F^{+,(n)}$, set $G^{(n)}_b= 0$.
Then each $G^{(n)}_b$ is continuous and globally Lipschitz, with
$\mathrm{Lip}\big(G^{(n)}_b\big)=0$ for $b\notin F^{+,(n)}$ and
$\mathrm{Lip}\big(G^{(n)}_b\big)\le C_bL_b$ for $b\in F^{+,(n)}$.

Item~1 holds because $G^{(n)}_b=\widehat G_b=G^{\mathrm{aux}}_b$ on $S_b=K^{+,(n)}_b$, and so the
$\delta/2$–approximation from Step~5 transfers verbatim and is $<\delta$.
Item~2 is by construction.
For Item~3, define $L_*^{(n)}=\max_{b\in F^{+,(n)}}\mathrm{Lip}\big(G^{(n)}_b\big)$ (finite since $F^{+,(n)}$ is finite).
Then for each object $a$,
\[
\operatorname*{ess\,sup}_{u\in I(a)}\ \mathrm{Lip}\big(G^{(n)}_{s(u)}\big)
\ \le\ \max\!\Big(\,0,\ \max_{\,b\in F^{+,(n)}\cap\{s(u):\,u\in I(a)\}}\ \mathrm{Lip}\big(G^{(n)}_b\big)\,\Big)
\ \le\ L_*^{(n)}.
\]
Thus $G^{(n)}$ is arrow–bundle uniformly Lipschitz, and the componentwise
lift $(G^{(n)})_{\downarrow}$ is admissible. 
\end{proof}

\begin{lemma}[Approximation of $\Phi$]\label{lem:approx-Phi}
Let $\Phi\in\EqvCont(X,Y)$ and $\varepsilon>0$. There exists $\Psi\in\CENN_\alpha(X,Y)$ with
\[
  \|\Psi-\Phi\|_{(K_a),F}
  := \max_{a\in F}\ \sup_{x\in K_a}\ \|(\Psi_a-\Phi_a)(x)\|_\infty
  \;<\; \varepsilon .
\]
\end{lemma}

\begin{proof}
Fix a finite $F\subseteq\Obj\,\C$ and compacts $K_a\subseteq X(a)$ for $a\in F$.
By (EC), there is a continuous natural retraction
$R:Y^{\downarrow}\Rightarrow Y$ with $R\circ\Delta_Y=\Id_Y$, realized by an
arrow–bundle convolution $L^{\downarrow}_{\mathsf R}$ whose kernel satisfies \textup{(C)}, \textup{(IN$^\downarrow$)}
and an $L^1$–bound $\sup_{y\in\Omega(a)}\|R(u,y)\|\le G_a^{\mathsf R}(u)$ with $G_a^{\mathsf R}\in L^1(\mu_a)$.
Set
\[
  M_R(F):=\max_{a\in F}\ \int_{I(a)} G_a^{\mathsf R}(u)\,d\mu_a(u)<\infty .
\]
Write the transport bounds explicitly as
\[
H_a^X := \operatorname*{ess\,sup}_{u\in I(a)}\|L^X_u\|<\infty,\qquad
H_a^Y := \operatorname*{ess\,sup}_{u\in I(a)}\|L^Y_u\|<\infty .
\]
Let also $S_a:=\sup_{x\in K_a}\|x\|_\infty$ and $C_a:=\sup_{x\in K_a}\|\Phi_a(x)\|_\infty<\infty$.

\emph{(i) Approximation on transported compacts.}
Choose increasing compacts $U_{a,n}\uparrow\supp\mu_a$ in $I(a)$ and put
\[
  F^{+,(n)}:=\{\,s(u)\mid u\in U_{a,n},\ a\in F\,\},\qquad
  K_{b,+}^{(n)}:=\bigcup_{a\in F}\ \bigcup_{\substack{u\in U_{a,n}\\ s(u)=b}} X(u)(K_a)\subset X(b).
\]
By Lemma~\ref{lem:cutoff}, for every $\delta>0$ and $n$ there exists an
objectwise continuous family $G^{(n)}=\{G^{(n)}_b:X(b)\to Y(b)\}$ such that
\[
\sup_{x\in K_{b,+}^{(n)}} \|G^{(n)}_b(x)-\Phi_b(x)\|_\infty<\delta \ \ (b\in F^{+,(n)}),\qquad
G^{(n)}_b\equiv 0 \ \ (b\notin F^{+,(n)}),
\]
and $G^{(n)}$ is arrow–bundle uniformly Lipschitz.

\emph{Uniform clipping.}
Set $B^\star:=1+\max_{a\in F}H_a^Y\,C_a$ and define the metric projection
onto the closed ball of radius $B^\star$ in $E_Y(b)$,
\[
\operatorname{proj}_{B^\star}(v):=
\begin{cases}
v,& \|v\|\le B^\star,\\[2pt]
\dfrac{B^\star}{\|v\|}\,v,& \|v\|>B^\star.
\end{cases}
\]
Replace each $G^{(n)}_b$ by $\widetilde G^{(n)}_b:=\operatorname{proj}_{B^\star}\!\circ G^{(n)}_b$.
Then $\|\widetilde G^{(n)}_b(x)\|_\infty\le B^\star$ for all $x$, and projection is
$1$–Lipschitz, so arrow–bundle uniform Lipschitzness is preserved.
Since
\[
\sup_{x\in K_{b,+}^{(n)}}\|\Phi_b(x)\|_\infty
\ \le\ \max_{a\in F}\ \sup_{u\in U_{a,n}}\ \sup_{x_a\in K_a}\ \|Y(u)\Phi_a(x_a)\|_\infty
\ \le\ \max_{a\in F}H_a^Y\,C_a\ <\ B^\star,
\]
and $\|G^{(n)}_b-\Phi_b\|<\delta\le 1$ on $K^{(n)}_{b,+}$, the projection is
inactive there; hence
\[
\sup_{x\in K_{b,+}^{(n)}} \|\widetilde G^{(n)}_b(x)-\Phi_b(x)\|_\infty<\delta
\quad (b\in F^{+,(n)}).
\]
For notational simplicity, write $G^{(n)}$ for the projected family below.

\emph{(ii) Compile to an equivariant map.}
Define
\[
  \Psi \ :=\ L^{\downarrow}_{\mathsf R} \circ (G^{(n)})_{\downarrow}\circ \Delta_X \ :\ X \Longrightarrow Y .
\]
Each factor is an admissible CENN layer; hence $\Psi\in\CENN_\alpha(X,Y)$ and, by naturality of the three factors, $\Psi\in\EqvCont(X,Y)$.

\emph{(iii) Localized error bound and choice of parameters.}
Apply Lemma~\ref{lem:stab-uniform} with $G=G^{(n)}$, $H=\Phi$ and the sets $U_{a,n}$:
\[
\|\Psi-\Phi\|_{(K_a),F}
\ \le\
\max_{a\in F}\Big\{A_a(G^{(n)},\Phi;U_{a,n})+T_a(G^{(n)},\Phi;U_{a,n})\Big\}.
\]
For the inside term, by construction of $G^{(n)}$,
\[
A_a(G^{(n)},\Phi;U_{a,n})
\ \le\
\int_{U_{a,n}}G_a^{\mathsf R}(u)\,\delta\,d\mu_a(u)
\ \le\ \delta\,M_R(F).
\]
For the tail term, using the projection bound and naturality of $\Phi$, for $\mu_a$–a.e.\ $u\in I(a)$,
\[
\sup_{x\in X(u)(K_a)}\|G^{(n)}_{s(u)}(x)\|_\infty \le B^\star,
\qquad
\sup_{x\in X(u)(K_a)}\|\Phi_{s(u)}(x)\|_\infty \le H_a^Y\,C_a,
\]
hence
\[
\sup_{x\in X(u)(K_a)}\|G^{(n)}_{s(u)}(x)-\Phi_{s(u)}(x)\|_\infty
\ \le\ \widehat B_a:=B^\star+H_a^Y C_a
\quad\text{for $\mu_a$–a.e.\ $u$.}
\]
Therefore
\[
T_a(G^{(n)},\Phi;U_{a,n})
\ \le\ \widehat B_a\!\!\int_{I(a)\setminus U_{a,n}}\!\!\! G_a^{\mathsf R}(u)\,d\mu_a(u).
\]
Let $t_n:=\max_{a\in F}\int_{I(a)\setminus U_{a,n}} G_a^{\mathsf R}\,d\mu_a$; since $G_a^{\mathsf R}\in L^1(\mu_a)$ and $U_{a,n}\uparrow\supp\mu_a$, we have $t_n\downarrow 0$.
Write $\widehat B:=\max_{a\in F}\widehat B_a$ and pick $n$ so large that $\widehat B\,t_n<\varepsilon/2$.
Finally choose $\delta:=\min\{\varepsilon/(2\,M_R(F)),\,1\}$.
Combining the bounds yields
\[
\|\Psi-\Phi\|_{(K_a),F}\ \le\ \delta\,M_R(F)+\widehat B\,t_n\ <\ \frac{\varepsilon}{2}+\frac{\varepsilon}{2}\ =\ \varepsilon.
\]
This completes the proof.
\end{proof}


\subsection{Proof Step 7 (last step)}

Let $\alpha:\mathbb{R}\to\mathbb{R}$ be a nonpolynomial, continuous, globally Lipschitz activation. We want to prove: For every $\Phi\in \EqvCont(X,Y)$, every finite $F\subset\Obj\,\C$, every family of compacts $K_a\subset X(a)$ $(a\in F)$, and every $\varepsilon>0$, there exists $\Psi\in \CENN_\alpha(X,Y)$ such that
\[
\|\Psi-\Phi\|_{(K_a),F}<\varepsilon.
\]

\begin{proof}
Fix $\Phi\in \EqvCont(X,Y)$, a finite $F\subset\Obj\,\C$, compacts $K_a\subset X(a)$ $(a\in F)$, and $\varepsilon>0$.
By Step~6, applied to this $(F,(K_a),\varepsilon)$, there exists a CENN
$\Psi\in \CENN_\alpha(X,Y)$ with
\[
\max_{a\in F}\ \sup_{x_a\in K_a}\ \|\Psi_a(x_a)-\Phi_a(x_a)\|_\infty\ <\ \varepsilon,
\]
i.e. $\Psi\in \mathcal U(\Phi;F,(K_a),\varepsilon)$.
Since $\mathcal U(\Phi;F,(K_a),\varepsilon)$ is a subbasic neighbourhood of $\Phi$ in the compact–open,
finite–object topology and since $F,(K_a),\varepsilon$ were arbitrary, it follows that every neighbourhood
of $\Phi$ intersects $\CENN_\alpha(X,Y)$.
Therefore $\CENN_\alpha(X,Y)$ is dense in $\EqvCont(X,Y)$.
\end{proof}

We have thus proven that 
$\CENN_\alpha(X,Y)$ is dense in $\EqvCont(X,Y)$
in the compact–open, finite–object topology.


\section{Specializations for Continuous and Discrete Structures}\label{sec:examples}

We discuss several examples to illustrate the general framework developed above: groups/groupoids, posets/lattices, and graphs/sheaves.

\subsection{Groups/Groupoids}\label{subsec:gmg}

We provide the continuous–track specializations of \Cref{def:natk,def:catconv} to compact groups and groupoids, together with the corresponding equivariant UATs as corollaries of \Cref{thm:density}. 

\paragraph{Groups as one-object categories.}
Let $G$ be a second–countable compact Hausdorff group with left Haar measure $\mu$ (or a discrete group with counting measure). Consider the one–object category $\C_G$ with $\Obj\,\C_G=\{\ast\}$ and $\Hom_{\C_G}(\ast,\ast)=G$, composition given by multiplication. Fix a compact base $\Omega\subset\mathbb{R}^{m}$ with a continuous left $G$–action $G\times\Omega\to\Omega$, $(g,y)\mapsto g\cdot y$. Choose continuous representations (fiber transports) $\rho_X,\rho_Y:G\to \mathrm{GL}(E_X),\mathrm{GL}(E_Y)$ and take
\[
\tau_g(y):=g^{-1}\!\cdot y,\qquad \pi_g(y):=g\cdot y,\qquad
L^X_g:=\rho_X(g^{-1}),\quad L^Y_g:=\rho_Y(g^{-1}).
\]
Then $\tau_{hg}\circ\pi_h=\tau_g$, $\tau_{\mathrm{id}}=\pi_{\mathrm{id}}=\mathrm{id}$ and $L^X_{hg}=L^X_g\circ L^X_h$ (and similarly for $Y$), so \Cref{feature} holds. A \emph{group kernel} is a map
\[
\mathsf K:\ G\times\Omega\longrightarrow \mathcal L(E_X,E_Y),\qquad (g,y)\mapsto \mathsf K(g,y),
\]
satisfying Carathéodory regularity and the $L^1$–bound (with respect to~$\mu$). The associated category convolution \eqref{eq:catconv} specializes to the steerable $G$–channel mixing layer\footnote{Note that we give fiberwise steerable mixing here and the spatial mixing convolution in the action groupoid specialization below.} 
\begin{equation}\label{eq:group-conv}
(\widetilde L_{\mathsf K}x)(y)
\;=\; b(y)\;+\;\int_G \mathsf K(g,y)\,\rho_X(g^{-1})\,x(y)\,d\mu(g),
\end{equation}
and the integrated naturality identity \eqref{eq:natk-integrated} becomes, after the change of variables $u'=hg$ and left–invariance of~$\mu$,
\begin{equation}\label{eq:group-steerability}
\int_G\Big(\rho_Y(h^{-1})\,\mathsf K(hg,\,h\!\cdot\! y)-\mathsf K(g,\,y)\Big)\,\rho_X(g^{-1})\,v\; d\mu(g)=0 \ \text{ for $h\in G$, $y\in\Omega$, $v\in E_X$.}
\end{equation}
In particular, the steerability law
\begin{equation}\label{eq:group-steerability-pointwise}
\rho_Y(h^{-1})\,\mathsf K(hg,\,h\!\cdot\! y)\;=\;\mathsf K(g,\,y)\qquad\text{for $\mu$–a.e.\ $g$ and all $h\in G$, $y\in\Omega$}
\end{equation}
suffices for \eqref{eq:group-steerability}, hence for equivariance of \eqref{eq:group-conv}. The shrinking–support approximate identity condition (SI) holds on $G$, and the retraction $R:Y_{\downarrow}\Rightarrow Y$ required in (EC) can be given by Haar averaging \cite{FollandAHA}, so all assumptions of \Cref{thm:density} hold.\footnote{As a remark concerning Step~2 of the UAT proof, note that in the
one–object compact–group case with $\tau_g(y)=g^{-1}\!\cdot y$ and
$\pi_g(y)=g\cdot y$, we instantiate $\sigma_g=\mathrm{id}$, and realize the
orbit–probe carriers $(x,y)\mapsto\langle \ell,(X(u_0)x)(y)\rangle$ by the
admissible block
$\Delta_X \!\to X^{\downarrow}(u_0)\!\to L^{\downarrow}_{\mathsf K^{(\varepsilon)}}$
with $\kappa^{(\varepsilon)}\rightharpoonup\delta_{\mathrm{id}}$. The
arrow–evaluation separation used in Step~1 needs to be verified in each concrete 
setting. It is satisfied, for example,
whenever for each compact $K_a\subset X(a)$ and $y\in\Omega$ the orbit
$G\!\cdot y$ meets every set on which two elements of $K_a$ differ; in
particular, this holds when the $G$–action on $\Omega$ is orbit–dense 
on the support of $K_a$. This covers many standard symmetry
examples encountered in machine learning practice.}

\begin{corollary}[UAT for $G$–equivariant CENNs]\label{cor:uat-group}
Let $G$ be as above, with feature functors $X,Y$ defined by $(\Omega,\rho_X,\rho_Y)$. Then the class of depth–finite $G$–equivariant CENNs 
is dense in $\EqvCont(X,Y)$ in the compact–open, finite–object topology.
\end{corollary}

\paragraph{Groupoids (including action groupoids).}
Let $\mathcal G\rightrightarrows \mathcal G^{(0)}$ be a second–countable compact Hausdorff groupoid with a continuous left Haar system $\{\lambda^a\}_{a\in\mathcal G^{(0)}}$ (see, e.g., \cite{Renault,Paterson,ADRenault}). View $\mathcal G$ as a small topological category with objects $\mathcal G^{(0)}$ and hom–sets $\Hom_{\mathcal G}(b,a)=\{u\in\mathcal G:\ s(u)=b,\ t(u)=a\}$; for each object $a$, let $I(a)=\bigsqcup_b\Hom_{\mathcal G}(b,a)$ and $\mu_a=\bigoplus_b\lambda^a\!\!\restriction_{\Hom(b,a)}$. Fix a family of compact bases $\Omega(a)\subset\mathbb R^{m_a}$ on which $\mathcal G$ acts by homeomorphisms through maps $u:b\!\to\! a\mapsto \tau_u:\Omega(a)\to\Omega(b)$, and set $\pi_u:=\tau_u^{-1}$. 
Choose fiber transports $L^X_u, L^Y_u$ as in \Cref{feature}. A \emph{groupoid kernel} is a family
\[
\mathsf K_{b\to a}:\ \Hom_{\mathcal G}(b,a)\times \Omega(a)\longrightarrow \mathcal L(E_X(b),E_Y(a))
\]
with the same Carathéodory and $L^1$ properties as in \Cref{def:natk}. The corresponding groupoid convolution reads
\begin{equation}\label{eq:groupoid-conv}
(\widetilde L_{\mathsf K}x)_a(y)
\;=\; b_a(y)\;+\;\int_{u\in I(a)}\!\!\mathsf K_{s(u)\to a}(u,y)\,\bigl(X(u)x_a\bigr)\!\bigl(\tau_u y\bigr)\,d\mu_a(u),
\end{equation}
and \eqref{eq:natk-integrated} is precisely the Haar–system version of integrated naturality:
\[
L^Y_w\!\int_{I(c)}\!\mathsf K(u',\pi_w y)\,(X(u')x_c)(\tau_{u'}\pi_w y)\,d\mu_c(u') = \int_{I(a)}\!\mathsf K(u,y)\,(X(w\circ u)x_c)(\tau_u y)\,d\mu_a(u).
\]
The (SI) condition holds on groupoids (approximate identities on target fibers), and the retraction $R:Y_{\downarrow}\Rightarrow Y$ is realized by an arrow–bundle convolution obtained by Haar averaging \cite{Renault}. Consequently, the general UAT specializes to groupoids:

\begin{corollary}[UAT for groupoid–equivariant CENNs]\label{cor:uat-groupoid}
Let $\mathcal G$ be a second–countable compact Hausdorff groupoid with a continuous left Haar system and $X,Y$ feature functors as above. Then depth–finite $\mathcal G$–equivariant CENNs are dense in $\EqvCont(X,Y)$ in the compact–open, finite–object topology. 
\end{corollary}

Note that the classical Cohen–Welling homogeneous-space convolution
$$(L_\kappa x)(y) = \int_G \kappa(g)\,x(g^{-1}\!\cdot y)\,d\mu(g)$$
is recovered as the special case where $\mathcal{C} = G \ltimes \Omega$ is
the action groupoid of a left $G$–action on $\Omega$, with objects
$\Omega$, arrows $(g,y)\colon y \to g\cdot y$, constant fibres
$X(y)=V_{\mathrm{in}}$, $Y(y)=V_{\mathrm{out}}$, and kernel
$K_{(g,y)} = \kappa(g)$.  In this case the arrow–bundle convolution 
reduces exactly to the standard
group-equivariant convolution, so the groupoid formulation 
generalizes the usual group case.

Note also that, in the discrete (counting–measure) setting, the integrals above are countable sums (absolutely convergent by (L1)) and integrated naturality reduces to the familiar incidence/compatibility equations on arrows.

\subsection{Posets/Lattices}\label{subsec:posets}

We now specialize to thin categories coming from finite posets and lattices.
Let $(P,\leq)$ be a finite poset with a least element $\bot$ and write $\C_P$ for its thin category:
$\Obj\,\C_P=P$ and
\[
\Hom_{\C_P}(d,a)=
\begin{cases}
\{d\!\to\! a\} & \text{if } d\le a,\\
\varnothing & \text{otherwise,}
\end{cases}
\]
with the unique composition induced by transitivity.

\paragraph{Base locality for thin categories.}
In this subsection we take a single compact base $\Omega=\{\star\}$ and identify $X(a)=C(\Omega,E_X(a))\cong E_X(a)$,
$Y(a)\cong E_Y(a)$ for all $a\in P$. We work in the discrete track: every hom–set carries counting measure, and for all arrows $u:d\to a$ we take
\[
\tau_u=\mathrm{id}_{\Omega},\qquad \pi_u=\mathrm{id}_{\Omega}.
\]
Thus $X(u)(x_a)=L^X_{u}\,x_a$ and $Y(u)(y_a)=L^Y_{u}\,y_a$.
In thin/discrete settings, the transports are typically restrictions/inclusions along the incidence of $P$, hence nonexpansive and uniformly bounded.
With $\Omega=\{\star\}$, all formulas below drop the base variable.

\paragraph{Poset–equivariant convolution.}
For $a\in P$ put $I(a)=\{\,d\in P: d\le a\,\}$.
A category kernel (\Cref{def:natk}) is here a family of linear maps
\[
\mathsf K_{d\to a}\in \mathcal L \big(E_X(d),E_Y(a)\big)\qquad(d\le a),
\]
and the category convolution (\Cref{def:catconv}) reduces to the finite sum
\begin{equation}\label{eq:poset-conv}
(\widetilde L_{\mathsf K}x)_a
\;=\;
b_a\;+\;\sum_{d\le a}\ \mathsf K_{d\to a}\,L^X_{d\to a}\,x_a,
\qquad a\in P,
\end{equation}
with $b_a\in E_Y(a)$ a natural bias.

\paragraph{Discrete integrated naturality.}
Since $\tau=\pi=\mathrm{id}$ and $\mu$ is counting measure, the integrated naturality
identity \textup{(IN)} in \Cref{def:natk} is equivalent to the operator equality
\begin{equation}\label{eq:poset-IN}
L^Y_{a\to c}\,\sum_{d\le c}\ \mathsf K_{d\to c}\,L^X_{d\to c}
\;=\;
\sum_{d\le a}\ \mathsf K_{d\to a}\,L^X_{d\to a}\,L^X_{a\to c}
\qquad(\forall\,a\le c).
\end{equation}
Equivalently, if we define the incidence operator
\(
\Sigma_a:=\sum_{d\le a}\mathsf K_{d\to a}\,L^X_{d\to a}\in\mathcal L(E_X(a),E_Y(a)),
\)
then \eqref{eq:poset-IN} is exactly the \emph{naturality} of the family $\{\Sigma_a\}_{a\in P}$:
\begin{equation}\label{eq:poset-Sigma-nat}
L^Y_{a\to c}\,\Sigma_c\;=\;\Sigma_a\,L^X_{a\to c}\qquad(\forall\,a\le c).
\end{equation}
Thus, in posets, enforcing equivariance amounts to parameterizing $\Sigma_a$ as a natural field of linear maps.

\paragraph{UAT for posets/lattices.}
With $\Omega=\{\star\}$ the carrier algebra in Step~1 of the UAT acts objectwise and separates points in each $K_a\subset E_X(a)$ (take $u=\mathrm{id}_a$).
In the finite case (counting measure; $\tau=\pi=\mathrm{id}$), the shrinking–support condition (SI) holds with Dirac masses. For the compilation step, we use the assumption (EC$_{\mathrm{thin}}$): 
There exist $Y$–extensions $E^Y_{d\to a}:E_Y(d)\to E_Y(a)$ such that, for all $a\le c$ and $d\in P$,
\[
\begin{aligned}
&\text{\rm(Nat)}\qquad &&L^Y_{a\to c}\,E^Y_{d\to c}=E^Y_{d\to a}\quad\text{if }d\le a,\\
&\text{\rm(Ann)}\qquad &&L^Y_{a\to c}\,E^Y_{d\to c}=0\quad\ \ \ \text{if }d\not\le a,\\
&\text{\rm(PoI)}\qquad &&\sum_{d\le a}E^Y_{d\to a}\,L^Y_{d\to a}=\Id_{E_Y(a)}.
\end{aligned}
\]
Equivalently, there exist spaces $U_Y(d)$ with a lower–set direct–sum splitting
\[
E_Y(a)\cong\bigoplus_{e\le a}U_Y(e)
\]
so that $L^Y_{d\to a}$ is the truncation and $E^Y_{d\to a}$ the inclusion.\footnote{In the poset/lattice setting, (EC$_\mathrm{thin}$)
holds, for example, when the target functor $Y$ is chosen to be
lower–set graded: for each $d\in P$ one fixes a finite-dimensional
local feature space $U_Y(d)$ and sets
$E_Y(a) := \bigoplus_{e\leq a} U_Y(e)$ for every $a\in P$, with
$L^Y_{d\to a}$ given by the canonical truncation
$E_Y(a)\to E_Y(d)$ and $E^Y_{d\to a}$ the inclusion of the $d$–summand.
In this case, the conditions \textup{(Nat)}, \textup{(Ann)}, and
\textup{(PoI)} are automatically satisfied, and the compilation
retraction $R:Y_{\downarrow}\Rightarrow Y$ is realized by the
arrow–bundle convolution
\(
  (L^{\downarrow}_{\mathsf R}x)_a
  \;=\; \sum_{d\leq a} E^Y_{d\to a}\, x_{d\to a},
\)
i.e.\ by inserting each arrow–wise component into its corresponding
summand and summing over $d\leq a$. Architecturally, this covers 
poset-/lattice–equivariant models in which the feature at 
level $a$ is obtained by stacking feature blocks indexed by all
lower elements $e\leq a$.}
Under \textup{(EC$_{\mathrm{thin}}$)}, the arrow–bundle convolution
\[
(L^{\downarrow}_{\mathsf R} x)_a\;:=\;\sum_{d\le a}E^Y_{d\to a}\,x_{d\to a}
\]
defines a continuous natural retraction $R:Y^\downarrow\Rightarrow Y$ with $R\circ\Delta_Y=\mathrm{id}_Y$.

Consequently, \Cref{thm:density} applies verbatim under (EC$_{\mathrm{thin}}$):

\begin{corollary}[UAT for poset/lattice–equivariant CENNs]\label{cor:uat-poset}
The class of depth–finite CENNs for a finite poset category satisfying (EC$_{\mathrm{thin}}$)
is dense in $\EqvCont(X,Y)$ in the compact–open, finite–object topology.
\end{corollary}

\paragraph{Remarks.}
As summarized in \eqref{eq:poset-Sigma-nat}, in the thin (poset) case
equivariance of the convolution layer associated with a kernel $K$ is
equivalent to the naturality of its incidence operators
$\Sigma_a := \sum_{d\le a} K_{d\to a}\,L^X_{d\to a}$.
Conversely, given any natural family $(\Sigma_a)_a$, a simple convenient
realization by a kernel is obtained by defining
$$\widetilde K_{a\to a}:=\Sigma_a \ \text{ and } \ \widetilde K_{d\to a}:=0 \ \ (d<a).$$




\subsection{Sheaves/Graphs via face-category convolution}\label{subsec:graphs-sheaves}

We treat a finite regular CW complex $K$ (graphs are the $1$–dimensional case) via its \emph{face category} $\C_K$: 
\[
\Obj\,\C_K \;=\; \{\hat{0}\}\ \cup\ \mathrm{Cells}(K), 
\]
with a unique arrow $\sigma\to\tau$ iff $\sigma\le\tau$ in the face poset (and $\hat{0}\le\tau$ for all cells $\tau$). Composition is induced by transitivity.  
We assume a common compact base $\Omega$ for all objects (in particular one may take $\Omega=\{*\}$), so the base transports are
\[
\tau_u=\mathrm{id}_\Omega,\qquad \pi_u=\mathrm{id}_\Omega.
\]
With counting measure on each hom–set, we define sheaf–style feature functors
\[
X(\tau)=C(\Omega,E_X(\tau)),\quad Y(\tau)=C(\Omega,E_Y(\tau)),
\]
for all $\tau\in\Obj\,\C_K$, and transports
\[
L^X_{\sigma\to\tau}:E_X(\tau)\to E_X(\sigma),\qquad L^Y_{\sigma\to\tau}:E_Y(\tau)\to E_Y(\sigma)\qquad(\sigma\le\tau),
\]
which are uniformly bounded (so the equation \eqref{eq:ess-bdd-LZ} holds). 
The bottom object $\hat{0}$ may be taken inert by setting $E_X(\hat{0})=E_Y(\hat{0})=0$ (no contribution in sums), or used as a global channel by choosing finite dimensions and providing $L^X_{\hat{0}\to\tau},L^Y_{\hat{0}\to\tau}$.

\paragraph{Sheaf/category convolution on cells.}
For $\tau\in\Obj\,\C_K$ let $I(\tau)=\{\sigma:\sigma\le\tau\}$.  
A category kernel is a Carathéodory family
\[
\mathsf K_{\sigma\to\tau}:\ \Omega\longrightarrow \mathcal L \big(E_X(\sigma),E_Y(\tau)\big)\qquad(\sigma\le\tau),
\]
and the category convolution \eqref{eq:catconv} reduces to the finite sum
\begin{equation}\label{eq:sheaf-conv}
(\widetilde L_{\mathsf K}x)_\tau(y)
\;=\;
b_\tau(y)\;+\;\sum_{\sigma\le\tau}\ \mathsf K_{\sigma\to\tau}(y)\,\big(L^X_{\sigma\to\tau}\,x_\tau(y)\big),
\qquad y\in\Omega,
\end{equation}
with the natural bias obeying $Y(\tau\to\gamma)b_\gamma=b_\tau$.  
Because $\tau=\pi=\mathrm{id}_\Omega$ and $\mu$ is counting measure, the integrated naturality (IN) identity becomes
\begin{equation}\label{eq:sheaf-IN}
L^Y_{\tau\to\gamma}\,\sum_{\sigma\le\gamma}\mathsf K_{\sigma\to\gamma}(y)\,L^X_{\sigma\to\gamma}
\;=\;
\sum_{\sigma\le\tau}\mathsf K_{\sigma\to\tau}(y)\,L^X_{\sigma\to\tau}\,L^X_{\tau\to\gamma}
\qquad(\forall\,\tau\le\gamma,\ y\in\Omega).
\end{equation}
Equivalently, with the \emph{incidence operator}
\[
\Sigma_\tau(y):=\sum_{\sigma\le\tau}\mathsf K_{\sigma\to\tau}(y)\,L^X_{\sigma\to\tau}:E_X(\tau)\longrightarrow E_Y(\tau),
\]
\eqref{eq:sheaf-IN} is precisely the naturality law
\begin{equation}\label{eq:sheaf-Sigma-nat}
L^Y_{\tau\to\gamma}\,\Sigma_\gamma(y)\;=\;\Sigma_\tau(y)\,L^X_{\tau\to\gamma}\qquad(\tau\le\gamma).
\end{equation}

\paragraph{Graphs as a $1$–complex.}
Let $G=(V,E)$ be a finite undirected graph, viewed as a $1$–dimensional CW complex. Then
\[
\Obj\,\C_G=\{\hat{0}\}\cup V\cup E,\qquad v\to e\ \text{iff}\ v\in\partial e,\qquad \hat{0}\to\tau\ \text{for all}\ \tau\in V\cup E.
\]
With $\Omega=\{*\}$, \eqref{eq:sheaf-conv} gives (writing $\partial e=\{v_1,v_2\}$)
\[
\begin{aligned}
(\widetilde L_{\mathsf K}x)_e
&= b_e \;+\; \mathsf K_{e\to e}\,x_e
\;+\;\sum_{v\in\partial e}\ \mathsf K_{v\to e}\,L^X_{v\to e}\,x_e
\;+\;\underbrace{\mathsf K_{\hat{0}\to e}\,L^X_{\hat{0}\to e}\,x_e}_{\text{global channel via }\hat{0}},\\[2pt]
(\widetilde L_{\mathsf K}x)_v
&= b_v \;+\; \mathsf K_{\hat{0}\to v}\,L^X_{\hat{0}\to v}\,x_v \;+\; \mathsf K_{v\to v}\,x_v
\qquad\big(I(v)=\{\hat{0},v\}\text{ in }\C_G\big).
\end{aligned}
\]

\paragraph{UAT for graphs and cellular sheaves.}
In finite graphs and finite regular CW complexes, we have counting measure, $\tau=\pi=\mathrm{id}_\Omega$, and (SI) holds with Dirac masses.  
Assuming the arrow–evaluation separation and the compilation condition for $Y$ (we clarify below when these hold), \Cref{thm:density} applies verbatim:

\begin{corollary}[UAT for graph/sheaf–equivariant CENNs]\label{cor:uat-graphs-sheaves}
For a finite regular CW complex with feature functors $X,Y$ induced by a cellular sheaf of Euclidean fibers and uniformly bounded transports along incidence satisfying (EC), the class of depth–finite CENNs is dense in $\EqvCont(X,Y)$ in the compact–open, finite–object topology.
\end{corollary}



The arrow–evaluation separation holds, for example, whenever the chosen base $\Omega$ together with the maps $\sigma_u$ ensure that the arrow–evaluations $$(x,y)\mapsto (X(u)x)(\sigma_u y)$$ separate distinct sections on the relevant compacts $K_a \subset X(a)$. This is automatic, in particular, when $\Omega = \{\ast\}$, which includes many standard sheaf- and message-passing architectures used in practice.

The compilation
assumption (EC) holds, for example, when for each cell $\tau$ the space
$Y(\tau)$ admits a finite direct-sum decomposition
$$Y(\tau)\cong\bigoplus_{\sigma\le\tau} U_Y(\sigma)$$ indexed by incident
faces $\sigma\le\tau$, and the structure maps
$$L^Y_{\sigma\to\tau}:Y(\tau)\to Y(\sigma)$$ are the canonical projections
onto the $\sigma$–summand. In this case one obtains a natural retraction
$R:Y_{\downarrow}\Rightarrow Y$ as an arrow–bundle convolution by
inserting each arrow-wise component into the corresponding summand and
summing over incident arrows, exactly as in the thin-category condition
(EC$_\mathrm{thin}$). This covers typical graph and cellular-sheaf
architectures where per-cell features are assembled by linear aggregation
of messages along incident cells.


\subsection{Sheaves/Graphs via neighbourhood categories}
\label{subsec:local-neighbourhood}

We provide another treatment of cellular sheaves and graphs. The idea is to work over the \emph{neighbourhood groupoid} of rooted \(k\)–patches. 
Let \(K\) be a finite regular CW complex with face category \(C_K\); as usual, graphs are the special case where \(K\) is a \(1\)–complex. Define a functor
\[
  X: C_K^{\op}\to\mathbf{Vect},\qquad X(\rho)=C(\Omega,E_X(\rho)),
\]
with linear restriction maps \(L^X_{\sigma\to\tau}:X(\tau)\to X(\sigma)\); similarly
\(Y: C_K^{\op}\to\mathbf{Vect}\) with fibers \(E_Y(\rho)\) and maps \(L^Y_{\sigma\to\tau}\).
We use the sup norm \(\|f\|_\infty=\sup_{y\in\Omega}\|f(y)\|\) on each \(C(\Omega,E)\).
Write
\[
  \mathcal{X}_{\mathrm{glob}}:=\textstyle\prod_{\rho\in\Obj(C_K)} X(\rho),\qquad
  \mathcal{Y}_{\mathrm{glob}}:=\textstyle\prod_{\rho\in\Obj(C_K)} Y(\rho).
\]

\paragraph{Rooted \(k\)–patches and the neighbourhood groupoid.}
Let \(\mathrm{dist}\) be the undirected graph distance in the Hasse diagram of \(C_K\)
(edges connect cover relations). For \(k\in\mathbb{N}\) and a root cell \(\tau\), set
\[
  V_k(\tau):=\{\rho\in\Obj(C_K):\ \mathrm{dist}(\rho,\tau)\le k\},
  \qquad B_k(\tau):=C_K\!\restriction_{V_k(\tau)}.
\]
A rooted \(k\)–isomorphism \(\phi:(B_k(\sigma),\sigma)\!\to\!(B_k(\tau),\tau)\)
is an isomorphism of categories preserving the root, incidence, and cell
dimensions/types.
Let \(\C^k_K\) be the small category with \(\Obj(\C^k_K)=\Obj(C_K)\) and with
\(\Hom_{\C^k_K}(\tau,\sigma)\) the set of rooted \(k\)–isomorphisms
\(\phi:(B_k(\sigma),\sigma)\!\to\!(B_k(\tau),\tau)\). Composition and identities are
those of rooted isomorphisms. Thus \(\C^k_K\) is a finite groupoid. In the graph case, 
this reduces to the rooted \(k\)–ball groupoid on vertices.

For each \(\tau\), let
\[
  I(\tau)\ :=\ \bigsqcup_{\sigma}\,\Hom_{\C^k_K}(\sigma,\tau)
\]
be the incoming–arrow bundle. 
We equip each hom–set \(\Hom_{\C^k_K}(\sigma,\tau)\) with the counting measure and hence \(I(\tau)\) with the induced counting measure \(\mu_\tau\). 
These are finite Radon measures; (NSP) is satisfied and (SI) holds by taking the discrete approximate identities
\[\kappa^{(\varepsilon)}_{\tau,u_0}:=\mathbf{1}_{\{u_0\}}\] 
for each \(u_0\in I(\tau)\).\footnote{$\mathbf{1}$ denotes the indicator function.}
Arrow–bundle convolutions therefore reduce to finite sums, and the measure-theoretic assumptions of the general CENN UAT are satisfied on \(\C^k_K\).

\paragraph{Type trivializations along patch isomorphisms.}
When stalk types vary, we transport types coherently along rooted isomorphisms.
Assume fiberwise trivializations: for each \(\phi:\tau\to\sigma\) in \(\C^k_K\) and \(\rho\in V_k(\tau)\), there are
linear isomorphisms
\[
  J^X_{\phi,\rho}: X(\phi^{-1}\rho)\xrightarrow{\ \cong\ }X(\rho),
  \qquad
  J^Y_{\phi,\tau}: Y(\sigma)\xrightarrow{\ \cong\ }Y(\tau),
\]
such that: (i) \(J^X_{\Id,\rho}=\Id\), \(J^Y_{\Id,\tau}=\Id\);
(ii) for composable \(\psi:w\to\sigma\), \(\phi:\tau\to w\),
\[
  J^X_{\psi\circ\phi,\rho}=J^X_{\phi,\rho}\circ J^X_{\psi,\phi^{-1}\rho},
  \quad
  J^Y_{\psi\circ\phi,\tau}=J^Y_{\phi,\tau}\circ J^Y_{\psi,w};
\]
(iii) for every arrow \(\mu\to\nu\) in \(B_k(\tau)\),
\[
  J^X_{\phi,\mu}\circ L^X_{\phi^{-1}\mu\to\phi^{-1}\nu}
  \ =\ L^X_{\mu\to\nu}\circ J^X_{\phi,\nu}.
\]
If \(E_X(\rho)=\mathbb{R}^{d_X}\) and \(E_Y(\rho)=\mathbb{R}^{d_Y}\), then
\(J^X_{\phi,\rho}=J^Y_{\phi,\tau}=\Id\) for all \(\phi,\rho,\tau\). 

\paragraph{Neighbourhood functors on \(\C^k_K\).}
Bundle all patch data into the root fiber and reindex contravariantly along
rooted isomorphisms.
Define \(X_k,Y_k:(\C^k_K)^{\op}\to\mathbf{Vect}\) by
\[
  X_k(\tau):=\prod_{\rho\in V_k(\tau)} X(\rho),\qquad
  Y_k(\tau):=Y(\tau).
\]
For \(\phi:\tau\to\sigma\) in \(\C^k_K\), set
\[
  \bigl(X_k(\phi)x\bigr)_\rho := J^X_{\phi,\rho}\bigl(x_{\phi^{-1}\rho}\bigr)
  \quad(\rho\in V_k(\tau)),\qquad
  Y_k(\phi):=J^Y_{\phi,\tau}.
\]
Under the above assumptions, \(X_k,Y_k\) are well-defined contravariant
functors: \(X_k(\psi\circ\phi)=X_k(\phi)\circ X_k(\psi)\) and
\(Y_k(\psi\circ\phi)=Y_k(\phi)\circ Y_k(\psi)\).

\paragraph{Integrated naturality (steerability).}
With the counting measures above and kernels supported on identity arrows, any 
arrow–bundle affine layer with local kernels \(K_\tau:X_k(\tau)\to Y_k(\tau)\) and biases \(b_\tau\in Y_k(\tau)\) reduces to a single identity term at each root \(\tau\).
The categorical equivariance condition is exactly the intertwining constraint
\[
  Y_k(\phi)\,K_\sigma\ =\ K_\tau\,X_k(\phi),\qquad 
  Y_k(\phi)\,b_\sigma\ =\ b_\tau
  \quad(\phi:\tau\to\sigma),
\]
which is precisely the steerability condition for convolutions on \(\C^k_K\).

\paragraph{Rooted \(k\)–locality and naturality.}
For \(H\in\mathcal{X}_{\mathrm{glob}}\) and root \(\tau\), define the local encoding
\[
  x_\tau(H):=(H_\rho)_{\rho\in V_k(\tau)}\in X_k(\tau).
\]
Collecting over roots yields the continuous map
\(H\mapsto x(H):=(x_\tau(H))_{\tau\in\Obj(\C^k_K)}\in\prod_{\tau}X_k(\tau)\).
A map \(F:\mathcal{X}_{\mathrm{glob}}\to\mathcal{Y}_{\mathrm{glob}}\) is
\emph{\(k\)–local} if for every root \(\tau\) there is a continuous map
\(f^{\mathrm{loc}}_\tau:X_k(\tau)\to Y_k(\tau)\) such that
\(F(H)(\tau)=f^{\mathrm{loc}}_\tau\!\bigl(x_\tau(H)\bigr)\) for all \(H\).
It is \emph{rooted \(k\)–local} if, for all \(\phi:\tau\to\sigma\) in \(\C^k_K\) and
\(a\in X_k(\sigma)\),
\begin{equation}
\label{eq:unified-rooted-equiv}
  f^{\mathrm{loc}}_\tau\!\bigl(X_k(\phi)a\bigr)\ =\ Y_k(\phi)\,f^{\mathrm{loc}}_\sigma(a).
\end{equation}
With \(X_k,Y_k\) as above, the assignment
\[
\begin{aligned}
  \{\,F: \text{rooted \(k\)–local}\,\}
  &\longleftrightarrow
  \{\, f: \text{natural transformation from } X_k \text{ to } Y_k \,\},\\
  \text{with}\quad
  f_\tau &= f^{\mathrm{loc}}_\tau, \qquad
  F(H)(\tau) = f_\tau\!\bigl(x_\tau(H)\bigr),
\end{aligned}
\]
is a bijection. Moreover, \eqref{eq:unified-rooted-equiv} is equivalent to the
naturality identity 
\[f_\tau\circ X_k(\phi)=Y_k(\phi)\circ f_\sigma\] 
for all
\(\phi:\tau\to\sigma\).\footnote{(\(\Rightarrow\)) Set \(f_\tau:=f^{\mathrm{loc}}_\tau\). Then \(f_\tau(X_k(\phi)a)=Y_k(\phi)f_\sigma(a)\) is exactly naturality. \; (\(\Leftarrow\)) Given a natural transformation \(f\) from $X_k$ to $Y_k$, define \(F(H)(\tau):=f_\tau(x_\tau(H))\). Then \(F\) is \(k\)–local by construction and satisfies
\eqref{eq:unified-rooted-equiv} by naturality. The constructions are inverse to each other.}

\paragraph{UAT for sheaves/graphs via neighbourhood groupoids.}
For the universal approximation result below, we work in the standard graph/sheaf setting taking \(\Omega=\{\ast\}\). In this case, the arrow–evaluation separation used in the categorical UAT holds trivially, and \(C(\Omega,E)\cong E\). On a finite groupoid equipped with counting measure, the compilation condition holds by averaging over each hom–set. Hence the general CENN UAT applies verbatim to \(\C^k_K\) with 
\(X_k,Y_k\):


\begin{corollary}[UAT for rooted \(k\)–local operators]
\label{thm:UAT-k-local}
Let
\(F:\mathcal{X}_{\mathrm{glob}}\to\mathcal{Y}_{\mathrm{glob}}\) be rooted \(k\)–local
with continuous local maps \(f^{\mathrm{loc}}_\tau\), and let
\(f\) be the associated natural transformation from \(X_k\) to \(Y_k\). Then for every
compact $K\subset\mathcal{X}_{\mathrm{glob}}$ and every \(\varepsilon>0\) there exists
\[
  \Phi\ \in\ \CENN_{\alpha}(X_k,Y_k)
\]
such that
\[
  \sup_{H\in K}\ \max_{\tau\in\Obj(\C^k_K)}\
  \bigl\|\,\Phi\bigl(x(H)\bigr)(\tau)\ -\ F(H)(\tau)\,\bigr\|_\infty\ <\ \varepsilon.
\]
\end{corollary}

The graph specialization is obtained as follows. If \(E_X(\rho)=\mathbb{R}^{d_X}\) and \(E_Y(\rho)=\mathbb{R}^{d_Y}\), then \(J=\Id\), \(X_k(\phi)\) is just coordinate reindexing by \(\phi^{-1}\) and \(Y_k(\phi)=\Id\). The correspondence between rooted \(k\)–local operators and natural transformations reduces to 
\[f_\tau(X_k(\phi)x)=f_\sigma(x)\] 
and Corollary~\ref{thm:UAT-k-local} specializes verbatim. With \(\Omega=\{\ast\}\) and \(K\) a \(1\)–complex, this is exactly the graph case.

\end{document}